\documentclass[twoside,11pt]{article}

\usepackage{jmlr2e}

\usepackage{geometry}

\usepackage[utf8]{inputenc} % allow utf-8 input
\usepackage[T1]{fontenc}    % use 8-bit T1 fonts
\usepackage{url}            % simple URL typesetting
\usepackage{nicefrac}       % compact symbols for 1/2, etc.
\usepackage{microtype}      % microtypography
\usepackage{graphicx}
\usepackage[font=small]{caption}
\usepackage{subcaption}
\usepackage{adjustbox}
\graphicspath{{figures/pdf/}}
\DeclareGraphicsExtensions{.pdf}

\usepackage{tikz}
\usetikzlibrary{bayesnet}

\usepackage{comment}

\usepackage{amsfonts}
\usepackage{amsmath}
\usepackage{amssymb}
\usepackage{algorithm}
\usepackage{algorithmic}
\usepackage{rotating}

\usepackage{booktabs}
\usepackage{multirow}
\usepackage{tabularx}
\usepackage{makecell}

\usepackage{enumitem}

\usepackage{macro}

\jmlrheading{18}{2017}{1-37}{10/16; Revised
6/17}{8/17}{16-498}{Maruan Al-Shedivat, Andrew Gordon Wilson, Yunus Saatchi, Zhiting Hu, Eric P. Xing}

\ShortHeadings{Learning Scalable Deep Kernels with Recurrent Structure}{Al-Shedivat, Wilson, Saatchi, Hu, Xing}
\firstpageno{1}

\begin{document}

\title{Learning Scalable Deep Kernels with Recurrent Structure}

\author{%
  \name{Maruan Al-Shedivat} \email{alshedivat@cs.cmu.edu} \\
  \addr{Carnegie Mellon University}
  \AND
  \name Andrew Gordon Wilson \email andrew@cornell.edu \\
  \addr Cornell University
  \AND
  \name Yunus Saatchi \email saatchi@cantab.net
  \AND
  \name Zhiting Hu \email zhitingh@cs.cmu.edu \\
  \addr Carnegie Mellon University
  \AND
  \name Eric P. Xing \email epxing@cs.cmu.edu \\
  \addr Carnegie Mellon University%
}

\editor{Neil Lawrence}

\maketitle

%!TEX root = ../16-498.tex

\begin{abstract}%
Many applications in speech, robotics, finance, and biology deal with sequential data, where ordering matters and recurrent structures are common.
However, this structure cannot be easily captured by standard kernel functions.
To model such structure, we propose expressive closed-form kernel functions for Gaussian processes.
The resulting model, GP-LSTM, fully encapsulates the inductive biases of long short-term memory (LSTM) recurrent networks, while retaining the non-parametric probabilistic advantages of Gaussian processes.
We learn the properties of the proposed kernels by optimizing the Gaussian process marginal likelihood using a new provably convergent semi-stochastic gradient procedure, and exploit the structure of these kernels for scalable training and prediction.
This approach provides a practical representation for Bayesian LSTMs.
We demonstrate state-of-the-art performance on several benchmarks, and thoroughly investigate a consequential autonomous driving application, where the predictive uncertainties provided by GP-LSTM are uniquely valuable.
\end{abstract}

%!TEX root = ../16-498.tex

\section{Introduction}\label{sec:introduction}

There exists a vast array of machine learning applications where the underlying datasets are sequential.
Applications range from the entirety of robotics, to speech, audio and video processing.
While neural network based approaches have dealt with the issue of \emph{representation learning} for sequential data, the important question of modeling and propagating uncertainty across time has rarely been addressed by these models.
For a robotics application such as a self-driving car, however, it is not just desirable, but essential to have complete predictive densities for variables of interest.
When trying to stay in lane and keep a safe following distance from the vehicle front, knowing the uncertainty associated with lanes and lead vehicles is as important as the point estimates.

Recurrent models with long short-term memory (LSTM)~\citep{hochreiter1997long} have recently emerged as the leading approach to modeling sequential structure.
The LSTM is an efficient gradient-based method for training recurrent networks.
LSTMs use a memory cell inside each hidden unit and a special gating mechanism that stabilizes the flow of the back-propagated errors, improving the learning process of the model.
While the LSTM provides state-of-the-art results on speech and text data~\citep{graves2013speech,sutskever2014sequence}, quantifying uncertainty or extracting full predictive distributions from deep models is still an area of active research~\citep{gal2016dropout}.

In this paper, we quantify the predictive uncertainty of deep models by following a Bayesian nonparametric approach.
In particular, we propose kernel functions which fully encapsulate the structural properties of LSTMs, for use with Gaussian processes.
The resulting model enables Gaussian processes to achieve state-of-the-art performance on \emph{sequential regression tasks}, while also allowing for a principled representation of uncertainty and non-parametric flexibility.
Further, we develop a provably convergent semi-stochastic optimization algorithm that allows mini-batch updates of the recurrent kernels.
We empirically demonstrate that this semi-stochastic approach significantly improves upon the standard non-stochastic first-order methods in
runtime and in the quality of the converged solution.
For additional scalability, we exploit the algebraic structure of these kernels, decomposing the relevant covariance matrices into Kronecker products of circulant matrices, for $\mathcal{O}(n)$ training time and $\mathcal{O}(1)$ test predictions~\citep{wilson2015msgp, wilson2015kissgp}.
Our model not only can be interpreted as a Gaussian process with a recurrent kernel, but also as a deep recurrent network with probabilistic outputs, infinitely many hidden units, and a utility function robust to overfitting.

Throughout this paper, we assume basic familiarity with Gaussian processes (GPs).
We provide a brief introduction to GPs in the background section; for a comprehensive reference, see, e.g., \citet{williams2006gaussian}.
In the following sections, we formalize the problem of learning from sequential data, provide background on recurrent networks and the LSTM, and present an extensive empirical evaluation of our model.
Specifically, we apply our model to a number of tasks, including system identification, energy forecasting, and self-driving car applications.
Quantitatively, the model is assessed on the data ranging in size from hundreds of points to almost a million with various signal-to-noise ratios demonstrating state-of-the-art performance and linear scaling of our approach.
Qualitatively, the model is tested on consequential self-driving applications: lane estimation and lead vehicle position prediction.
Indeed, the main focus of this paper is on achieving state-of-the-art performance on consequential applications involving sequential data, following straightforward and scalable approaches to building highly flexible Gaussian process.

We release our code as a library at: \url{http://github.com/alshedivat/keras-gp}.
This library implements the ideas in this paper as well as deep kernel learning \citep{wilson2016dkl} via a Gaussian process layer that can be added to \emph{arbitrary} deep architectures and deep learning frameworks, following the Keras API specification.
More tutorials and resources can be found at \url{https://people.orie.cornell.edu/andrew/code}.

%!TEX root = ../16-498.tex

\section{Background}\label{sec:background}

We consider the problem of learning a regression function that maps sequences to real-valued target vectors.
Formally, let $\overline{\Xv} = \{\overline{\xv_i}\}_{i=1}^n$ be a collection of sequences, $\overline{\xv_i} = [\xv^1_i, \xv^2_i, \cdots, \xv^{l_i}]$, each with corresponding length, $l_i$, where $\xv^j_i \in \Xc$, and $\Xc$ is an arbitrary domain.
Let $\yv = \{\yv_i\}_{i=1}^n$, $\yv_i \in \Rb^d$, be a collection of the corresponding real-valued target vectors.
Assuming that only the most recent $L$ steps of a sequence are predictive of the targets, the goal is to learn a function, $f : \Xc^L \mapsto \Rb^d$, from some family, $\Fc$, based on the available data.

As a working example, consider the problem of estimating position of the lead vehicle at the next time step from LIDAR, GPS, and gyroscopic measurements of a self-driving car available for a number of previous steps.
This task is a classical instance of the \emph{sequence-to-reals regression}, where a temporal sequence of measurements is regressed to the future position estimates.
In our notation, the sequences of inputs are vectors of measurements, $\overline{\xv_1} = [\xv^1],\, \overline{\xv_2} = [\xv^1, \xv^2],\, \dots,\, \overline{\xv_n} = [\xv^1, \xv^2, \cdots \xv^n]$, are indexed by time and would be of growing lengths.
Typically, input sequences are considered up to a finite-time horizon, $L$, that is assumed to be predictive for the future targets of interest.
The targets, $\yv_1, \yv_2,\, \dots,\, \yv_n$, are two-dimensional vectors that encode positions of the lead vehicle in the ego-centric coordinate system of the self-driving car.

Note that the problem of learning a mapping, $f : \Xc^L \mapsto \Rb^d$, is challenging.
While considering whole sequences of observations as input features is necessary for capturing long-term temporal correlations, it virtually blows up the dimensionality of the problem.
If we assume that each measurement is $p$-dimensional, i.e., $\Xc \subseteq \Rb^p$, and consider $L$ previous steps as distinct features, the regression problem will become $(L \times p)$-dimensional.
Therefore, to avoid overfitting and be able to extract meaningful signal from a finite amount of data, it is crucial to exploit the sequential nature of observations.

\textbf{Recurrent models.}
One of the most successful ways to exploit sequential structure of the data is by using a class of recurrent models.
In the sequence-to-reals regression scenario, such a model expresses the mapping $f : \Xc^L \mapsto \Rb^d$ in the following general recurrent form:
\begin{equation}
    \yv = \psiv(\hv^L) + \epsilonv^t,\ \hv^t = \phiv(\hv^{t-1}, \xv^t) + \deltav^t,\, t = 1, \dots, L,
\end{equation}
where $\xv^t$ is an input observation at time $t$, $\hv^t$ is a corresponding latent representation, and $\yv$ is a target vector.
Functions $\phiv(\cdot)$ and $\psiv(\cdot)$ specify model transitions and emissions, respectively, and $\deltav^t$ and $\epsilonv^t$ are additive noises.
While $\phiv(\cdot)$ and $\psiv(\cdot)$ can be arbitrary, they are typically time-invariant.
This strong but realistic assumption incorporated into the structure of the recurrent mapping significantly reduces the complexity of the family of functions, $\Fc$, regularizes the problem, and helps to avoid severe overfitting.

Recurrent models can account for various patterns in sequences by \emph{memorizing} internal representations of their dynamics via adjusting $\phiv$ and $\psiv$.
Recurrent neural networks (RNNs) model recurrent processes by using linear parametric maps followed by nonlinear activations:
\begin{equation}
    \yv = \psiv(W_{hy}^\top\hv^{t-1}),\ \hv^t = \phiv(W_{hh}^\top \hv^{t-1}, W_{xh}^\top \xv^{t-1}),\, t = 1, \dots, L,
\end{equation}
where $W_{hy}, W_{hh}, W_{xh}$ are weight matrices to be learned\footnote{The bias terms are omitted for clarity of presentation.} and $\phiv(\cdot)$ and $\psiv(\cdot)$ here are some fixed element-wise functions.
Importantly and contrary to the standard hidden Markov models (HMMs), \emph{the state} of an RNN at any time $t$ is \emph{distributed} and effectively represented by an entire hidden sequence, $[\hv^1, \cdots, \hv^{t-1}, \hv^t]$.
A major disadvantage of the vanilla RNNs is that their training is nontrivial due to the so-called \emph{vanishing gradient} problem~\citep{bengio1994learning}: the error back-propagated through $t$ time steps diminishes exponentially which makes learning long-term relationships nearly impossible.

\textbf{LSTM.}
To overcome vanishing gradients, \cite{hochreiter1997long} proposed a long short-term memory (LSTM) mechanism that places a \emph{memory cell} into each hidden unit and uses differentiable gating variables.
The update rules for the hidden representation at time $t$ have the following form (here $\sigmoid{\cdot}$ and $\tanhoid{\cdot}$ are element-wise sigmoid and hyperbolic tangent functions, respectively):\\[-1.5ex]
\begin{minipage}[t]{\textwidth}
    \begin{minipage}[t]{0.30\textwidth}
        \vspace{-4pt}
        \centering
        \includegraphics[width=\textwidth]{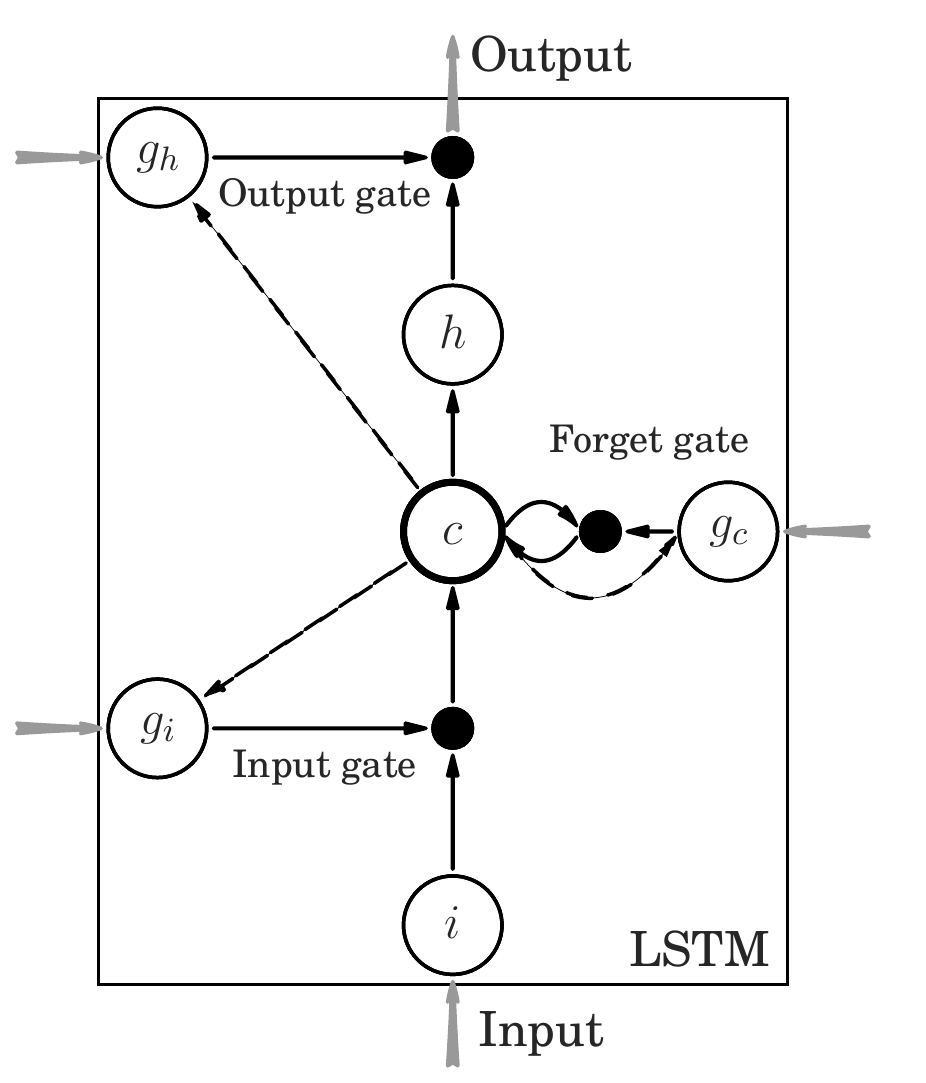}
    \end{minipage}%
    \begin{minipage}[t]{0.70\textwidth}
        \vspace{0pt}
        \begin{equation}
        \begin{aligned}
        \iv^t & = \tanhoid{W_{xc}^\top \xv^t + W_{hc}^\top \hv^{t-1} + \bv_c},\\
        \gv_i^t & = \sigmoid{W_{xi}^\top \xv^t + W_{hi}^\top \hv^{t-1} + W_{ci}^\top \cv^{t-1} + \bv_i},\\
        \cv^t & = \gv_c^t\ \cv^{t-1} + \gv_i^t\ \iv^t,\\
        \gv_c^t & = \sigmoid{W_{xf}^\top \xv^t + W_{hf}^\top \hv^{t-1} + W_{cf}^\top \cv^{t-1} + \bv_f},\\
        \ov^t & = \tanhoid{\cv^t},\\
        \gv_o^t & = \sigmoid{W_{xo}^\top \xv^t + W_{ho}^\top \hv^{t-1} + W_{co}^\top \cv^t + \bv_o},\\
        \hv^t & = \gv_o^t\ \ov^t.
        \end{aligned}
        \end{equation}
    \end{minipage}
\end{minipage}
As illustrated above, $\gv_i^t$, $\gv_c^t$, and $\gv_o^t$ correspond to the input, forget, and output gates, respectively.
These variables take their values in $[0, 1]$ and when combined with the internal states, $\cv^t$, and inputs, $\xv^t$, in a multiplicative fashion, they play the role of soft gating.
The gating mechanism not only improves the flow of errors through time, but also, allows the the network to decide whether to keep, erase, or overwrite certain memorized information based on the forward flow of inputs and the backward flow of errors.
This mechanism adds stability to the network's memory.

\textbf{Gaussian processes.}
The Gaussian process (GP) is a Bayesian nonparametric model that generalizes the Gaussian distributions to functions.
We say that a random function $f$ is drawn from a GP with a mean function $\mu$ and a covariance kernel $k$, $f \sim \GP{\mu, k}$, if for any vector of inputs, $[\xv_1, \xv_2, \dots, \xv_n]$, the corresponding vector of function values is Gaussian:
\begin{equation*}
  [f(\xv_1), f(\xv_2), \dots, f(\xv_n)] \sim \gau{\muv, K_{X,X}},
\end{equation*}
with mean $\muv$, such that $\muv_i = \mu(\xv_i)$, and covariance matrix $K_{X,X}$ that satisfies $(K_{X,X})_{ij} = k(\xv_i, \xv_j)$.
GPs can be seen as distributions over the reproducing kernel Hilbert space (RKHS) of functions which is uniquely defined by the kernel function, $k$~\citep{scholkopf2002learning}.
GPs with RBF kernels are known to be universal approximators with prior support to within an arbitrarily small epsilon band of any continuous function~\citep{micchelli2006universal}.

Assuming additive Gaussian noise, $y \mid \xv \sim \gau{f(\xv), \sigma^2}$, and a GP prior on $f(\xv)$, given training inputs $\xv$ and training targets $\yv$, the predictive distribution of the GP evaluated at an arbitrary test point $\xv_*$ is:
\begin{equation}
  \fv_* \mid \xv_*, \xv, \yv, \sigma^2 \sim \gau{\mathbb{E}[\fv_*], \Cov[\fv_*]},
\end{equation}
where
\begin{equation}
  \begin{aligned}
    \Expec[\fv_*] &= \muv_{X_*} + K_{X_*,X} [K_{X,X} + \sigma^2 I]^{-1}\yv,\\
    \Cov[\fv_*] &= K_{X_*,X_*} - K_{X_*,X} [K_{X,X} + \sigma^2 I]^{-1} K_{X,X_*}.
  \end{aligned}
\end{equation}
Here, $K_{X_*,X}$, $K_{X,X_*}$, $K_{X,X}$, and $K_{X_*,X_*}$ are matrices that consist of the covariance function, $k$, evaluated at the corresponding points, $\xv \in \Xv$ and $\xv_* \in \Xv_*$, and $\muv_{X_*}$ is the mean function evaluated at $\xv_* \in \Xv_*$.
GPs are fit to the data by optimizing \emph{the evidence}---the marginal probability of the data given the model---with respect to kernel hyperparameters.
The evidence has the form:
\begin{equation}
    \log \prob{\yv \mid \xv} = - \left[\yv^\top (K + \sigma^2 I)^{-1}\yv + \log \det (K + \sigma^2 I)\right] + \const,
\end{equation}
where we use a shorthand $K$ for $K_{X,X}$, and $K$ implicitly depends on the kernel hyperparameters.
This objective function consists of a \emph{model fit} and a \emph{complexity penalty} term that results in an automatic Occam's razor for realizable functions~\citep{rasmussen2001occam}.
By optimizing the evidence with respect to the kernel hyperparameters, we effectively learn the the structure of the space of functional relationships between the inputs and the targets.
For further details on Gaussian processes and relevant literature we refer interested readers to the classical book by~\cite{williams2006gaussian}.

Turning back to the problem of learning from sequential data, it seems natural to apply the powerful GP machinery to modeling complicated relationships.
However, GPs are limited to learning only pairwise correlations between the inputs and are unable to account for long-term dependencies, often dismissing complex temporal structures.
Combining GPs with recurrent models has potential to addresses this issue.

%!TEX root = ../16-498.tex

\section{Related work}\label{sec:related-work}

The problem of learning from sequential data, especially from temporal sequences, is well known in the control and dynamical systems literature.
Stochastic temporal processes are usually described either with generative \emph{autoregressive models} (AM) or with \emph{state-space models} (SSM)~\citep{van2012subspace}.
The former approach includes nonlinear auto-regressive models with exogenous inputs (NARX) that are constructed by using, \textit{e.g.}, neural networks~\citep{lin1996learning} or Gaussian processes~\citep{kocijan2005dynamic}.
The latter approach additionally introduces unobservable variables, \emph{the state}, and constructs autoregressive dynamics in the latent space.
This construction allows to represent and propagate uncertainty through time by explicitly modeling the signal (via the state evolution) and the noise.
Generative SSMs can be also used in conjunction with discriminative models via the Fisher kernel~\citep{jaakkola1999fisherkernel}.

Modeling time series with GPs is equivalent to using \textit{linear-Gaussian} autoregressive or SSM models~\citep{box1994time}.
Learning and inference are efficient in such models, but they are not designed to capture long-term dependencies or correlations beyond pairwise.
\citet{wang2005gaussian} introduced GP-based state-space models (GP-SSM) that use GPs for transition and/or observation functions.
These models appear to be more general and flexible as they account for uncertainty in the state dynamics, though require complicated approximate training and inference, which are hard to scale~\citep{turner2010state,frigola2014variational}.

Perhaps the most recent relevant work to our approach is \emph{recurrent Gaussian processes} (RGP)~\citep{mattos2015recurrent}.
RGP extends the GP-SSM framework to regression on sequences by using a recurrent architecture with GP-based activation functions.
The structure of the RGP model mimics the standard RNN, where every parametric layer is substituted with a Gaussian process.
This procedure allows one to propagate uncertainty throughout the network for an additional cost.
Inference is intractable in RGP, and efficient training requires a sophisticated \emph{approximation} procedure, the so-called \emph{recurrent variational Bayes}.
In addition, the authors have to turn to RNN-based approximation of the variational mean functions to battle the growth of the number of variational parameters with the size of data.
While technically promising, RGP seems problematic from the application perspective, especially in its implementation and scalability aspects.

Our model has several distinctions with prior work aiming to regress sequences to reals.
Firstly, one of our goals is to keep the model as simple as possible while being able to represent and quantify \emph{predictive uncertainty}.
We maintain an analytical objective function and refrain from complicated and difficult-to-diagnose inference schemes.
This simplicity is achieved by giving up the idea of propagating signal through a chain GPs connected in a recurrent fashion.
Instead, we propose to directly \emph{learn kernels} with recurrent structure via joint optimization of a simple functional composition of a standard GP with a recurrent model (\emph{e.g.}, LSTM), as described in detail in the following section.
Similar approaches have recently been explored and proved to be fruitful for non-recurrent deep networks~\citep{wilson2016dkl, wilson2016stochastic, calandra2016manifold}.
We remark that combinations of GPs with nonlinear functions have also been considered in the past in a slightly different setting of warped regression targets~\citep{snelson2003warped,wilson2010copula,lazaro2012bayesian}.
Additionally, uncertainty over the recurrent parts of our model is represented via dropout, which is computationally cheap and turns out to be equivalent to approximate Bayesian inference in a deep Gaussian process~\citep{damianou2013deep} with particular intermediate kernels~\citep{gal2016dropout,gal2016theoretically}.
Finally, one can also view our model as a standalone flexible Gaussian process, which leverages learning techniques that scale to massive datasets~\citep{wilson2015kissgp, wilson2015msgp}.

%!TEX root = ../16-498.tex

%!TEX root = ../16-498.tex

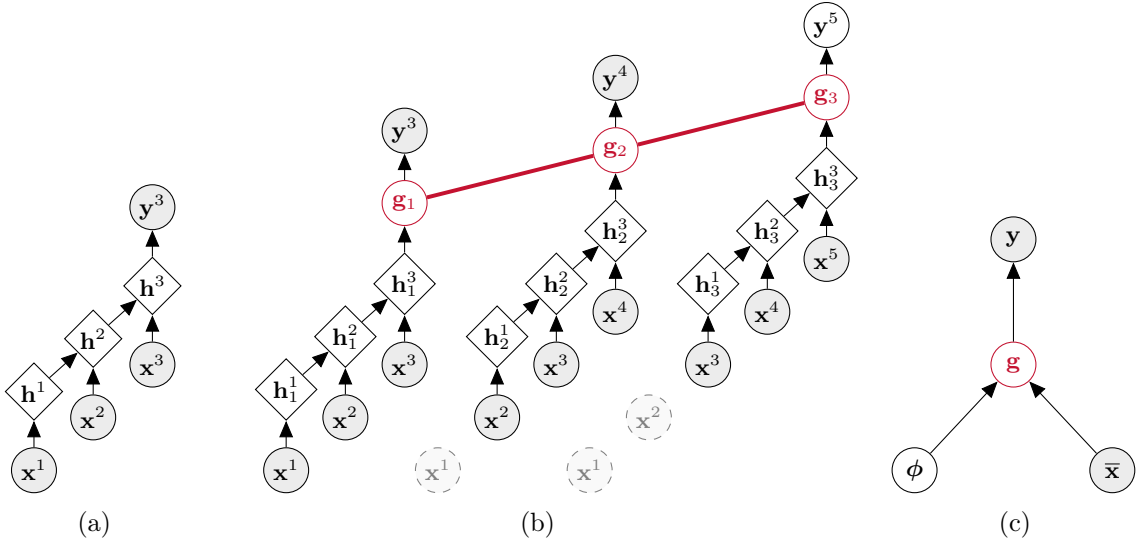
\begin{figure*}[t!]
% \vspace{-4ex}
\centering
\begin{subfigure}[b]{0.16\textwidth}
  \begin{tikzpicture}
    % Redefine some styles
    \tikzstyle{latent} = [circle,draw=black,inner sep=1pt,
    minimum size=17pt, font=\fontsize{9}{9}\selectfont, node distance=1]

    % Sample 1
    \node (x11) [obs] {$\xv^1$};
    \node (x12) [obs, right=5pt of x11, yshift=20pt] {$\xv^2$};
    \node (x13) [obs, right=5pt of x12, yshift=20pt] {$\xv^3$};

    \node (h11) [det, above=10pt of x11] {$\hv^1$};
    \node (h12) [det, above=10pt of x12] {$\hv^2$};
    \node (h13) [det, above=10pt of x13] {$\hv^3$};

    % \node (y11) [obs, dashed, draw=black!50, fill=gray!5, above=10pt of h11] {\textcolor{black!50}{$\yv^1$}};
    % \node (y12) [obs, dashed, draw=black!50, fill=gray!5, above=10pt of h12] {\textcolor{black!50}{$\yv^2$}};
    \node (y13) [obs, above=10pt of h13] {$\yv^3$};

    \edge {x11} {h11};
    \edge {x12} {h12};
    \edge {x13} {h13};
    \edge {h11} {h12};
    \edge {h12} {h13};
    \edge {h13} {y13};
  \end{tikzpicture}
  \caption{}\label{fig:lstm}
\end{subfigure}
\hfill
\begin{subfigure}[b]{0.50\textwidth}
  \begin{tikzpicture}
    % Redefine some styles
    \tikzstyle{latent} = [circle,draw=black,inner sep=1pt,
    minimum size=17pt, font=\fontsize{9}{9}\selectfont, node distance=1]

    % Sample 1
    \node (x11) [obs] {$\xv^1$};
    \node (x12) [obs, right=5pt of x11, yshift=20pt] {$\xv^2$};
    \node (x13) [obs, right=5pt of x12, yshift=20pt] {$\xv^3$};

    \node (h11) [det, above=10pt of x11] {$\hv^1_1$};
    \node (h12) [det, above=10pt of x12] {$\hv^2_1$};
    \node (h13) [det, above=10pt of x13] {$\hv^3_1$};

    \node (g13) [latent, draw=cmu_red,  above=10pt of h13] {\textcolor{cmu_red}{$\gv_1$}};

    \node (y13) [obs, above=10pt of g13] {$\yv^3$};

    \edge {x11} {h11};
    \edge {x12} {h12};
    \edge {x13} {h13};
    \edge {h11} {h12};
    \edge {h12} {h13};
    \edge {h13} {g13};
    \edge {g13} {y13};

    % Sample 2
    \node (x21) [obs, dashed, draw=black!50, fill=gray!5, right=40pt of x11] {\textcolor{black!50}{$\xv^1$}};
    \node (x22) [obs, right=5pt of x21, yshift=20pt] {$\xv^2$};
    \node (x23) [obs, right=5pt of x22, yshift=20pt] {$\xv^3$};
    \node (x24) [obs, right=5pt of x23, yshift=20pt] {$\xv^4$};

    \node (h21) [det, above=10pt of x22] {$\hv^1_2$};
    \node (h22) [det, above=10pt of x23] {$\hv^2_2$};
    \node (h23) [det, above=10pt of x24] {$\hv^3_2$};

    \node (g24) [latent, draw=cmu_red,  above=10pt of h23] {\textcolor{cmu_red}{$\gv_2$}};

    \node (y24) [obs, above=10pt of g24] {$\yv^4$};

    \edge {x22} {h21};
    \edge {x23} {h22};
    \edge {x24} {h23};
    \edge {h21} {h22};
    \edge {h22} {h23};
    \edge {h23} {g24};
    \edge {g24} {y24};

    % Sample 3
    \node (x31) [obs, dashed, draw=black!50, fill=gray!5, right=40pt of x21] {\textcolor{black!50}{$\xv^1$}};
    \node (x32) [obs, dashed, draw=black!50, fill=gray!5, right=5pt of x31, yshift=20pt] {\textcolor{black!50}{$\xv^2$}};
    \node (x33) [obs, right=5pt of x32, yshift=20pt] {$\xv^3$};
    \node (x34) [obs, right=5pt of x33, yshift=20pt] {$\xv^4$};
    \node (x35) [obs, right=5pt of x34, yshift=20pt] {$\xv^5$};

    \node (h31) [det, above=10pt of x33] {$\hv^1_3$};
    \node (h32) [det, above=10pt of x34] {$\hv^2_3$};
    \node (h33) [det, above=10pt of x35] {$\hv^3_3$};

    \node (g35) [latent, draw=cmu_red,  above=10pt of h33] {\textcolor{cmu_red}{$\gv_3$}};

    \node (y35) [latent, above=10pt of g35] {$\yv^5$};

    \edge {x33} {h31};
    \edge {x34} {h32};
    \edge {x35} {h33};
    \edge {h31} {h32};
    \edge {h32} {h33};
    \edge {h33} {g35};
    \edge {g35} {y35};

    % Gaussian random field
    \edge[-, draw=cmu_red, line width=1.5pt] {g13} {g24};
    \edge[-, draw=cmu_red, line width=1.5pt] {g24} {g35};
  \end{tikzpicture}
  \caption{}\label{fig:gp-lstm}
\end{subfigure}
\hfill
\begin{subfigure}[b]{0.22\textwidth}
  \begin{tikzpicture}
    % Redefine some styles
    \tikzstyle{latent} = [circle,draw=black,inner sep=1pt,
    minimum size=17pt, font=\fontsize{9}{9}\selectfont, node distance=1]

    \node (x) [obs] {$\overline{\xv}$};
    \node (g) [latent, draw=cmu_red, left=20pt of x, yshift=40pt] {\textcolor{cmu_red}{$\gv$}};
    \node (phi) [latent, left=20pt of g, yshift=-40pt] {$\phiv$};
    \node (y) [obs, above=30pt of g] {$\yv$};

    \edge {x} {g};
    \edge {g} {y};
    \edge {phi} {g};
  \end{tikzpicture}
  \caption{}\label{fig:gp-lstm-pgm}
\end{subfigure}
\caption{\small
(a) Graphical representation of a recurrent model (RNN/LSTM) that maps an input sequence to a target value in one-step-ahead prediction manner.
Shaded variables are observable, diamond variables denote deterministic dependence on the inputs.
(b) Graphical model for GP-LSTM with a time lag, $L = 3$, two training time points, $t = 3$ and $t = 4$, and a testing time point, $t = 5$.
Latent representations are mapped to the outputs through a Gaussian field (denoted in red) that globally correlates predictions.
Dashed variables represent data instances unused at the given time step.
(c) Graphical representation of a GP with a kernel structured with a parametric map, $\phiv$.%
}
\label{fig:graph_models}
\end{figure*}

\section{Learning recurrent kernels}\label{sec:recurrent-kernels}

Gaussian processes with different kernel functions correspond to different structured probabilistic models.
For example, some special cases of the Mat\'ern class of kernels induce models with Markovian structure~\citep{stein1999interpolation}.
To construct deep kernels with recurrent structure we transform the original input space with an LSTM network and build a kernel directly in the transformed space, as shown in Figure~\ref{fig:gp-lstm}.

In particular, let $L$ $\phiv : \Xc^L \mapsto \Hc$ be an arbitrary deterministic transformation of the input sequences into some latent space, $\Hc$.
Next, let $k : \Hc^2 \mapsto \Rb$ be a real-valued kernel defined on $\Hc$.
The decomposition of the kernel and the transformation, $\tilde k = k \circ \phiv$, is defined as
\begin{equation}
    \label{eq:recurrent_kernel}
    \tilde k(\overline{\xv}, \overline{\xv}^\prime) = k(\phiv(\overline{\xv}), \phiv(\overline{\xv}^\prime)), \text{ where } \overline{\xv}, \overline{\xv}^\prime \in \Xc^L, \text{ and } \tilde k : \left(\Xc^L\right)^2 \mapsto \Rb.
\end{equation}
It is trivial to show that $\tilde k(\overline{\xv}, \overline{\xv}^\prime)$ is a valid kernel defined on $\Xc^L$~\citep[Ch.~5.4.3]{mackay1998introduction}.
In addition, if $\phiv(\cdot)$ is represented by a neural network, the resulting model can be viewed as the same network, but with an additional GP-layer and the negative log marginal likelihood (NLML) objective function used instead of the standard mean squared error (MSE).

Input embedding is well-known in the Gaussian process literature~\citep[e.g.][]{mackay1998introduction,hinton2008using}.
Recently, \citet{wilson2016dkl, wilson2016stochastic} have successfully scaled the approach and demonstrated strong results in regression and classification tasks for kernels based on feedforward and convolutional architectures.
In this paper, we apply the same technique to learn kernels with recurrent structure by transforming input sequences with a recurrent neural network that acts as $\phiv(\cdot)$.
In particular, a (multi-layer) LSTM architecture is used to embed $L$ steps of the input time series into a single vector in the hidden space, $\Hc$.
For the embedding, as common, we use the last hidden vector produced by the recurrent network.
Note that however, any variations to the embedding (e.g., using other types of recurrent units, adding 1-dimensional pooling layers, or attention mechanisms) are all fairly straightforward\footnote{Details on the particular architectures used in our empirical study are discussed in the next section.}.
More generally, the recurrent transformation can be random itself (Figure~\ref{fig:graph_models}), which would enable direct modeling of uncertainty \emph{within} the recurrent dynamics, but would also require inference for $\phiv$~\citep[e.g., as in][]{mattos2015recurrent}.
In this study, we limit our consideration of random recurrent maps to only those induced by dropout.

Unfortunately, once the the MSE objective is substituted with NLML, it no longer factorizes over the data.
This prevents us from using the well-established stochastic optimization techniques for training our recurrent model.
In the case of feedforward and convolutional networks, \citet{wilson2016dkl} proposed to pre-train the input transformations and then fine-tune them by jointly optimizing the GP marginal likelihood with respect to hyperparameters and the network weights using \emph{full-batch} algorithms.
When the transformation is recurrent, stochastic updates play a key role.
Therefore, we propose a semi-stochastic block-gradient optimization procedure which allows mini-batching weight updates and fully joint training of the model from scratch.

\subsection{Optimization}\label{sec:optimization}
The negative log marginal likelihood of the Gaussian process has the following form:
\begin{equation}
  \lik(K) = \yv^\top (K_y + \sigma^2 I)^{-1}\yv + \log\det(K_y + \sigma^2 I) + \const,
\end{equation}
where $K_y + \sigma^2 I$ ($\overset{\Delta}{=} K$) is the Gram kernel matrix, $K_y$, is computed on $\{\phiv(\overline{\xv}_i)\}_{i=1}^N$ and implicitly depends on the base kernel hyperparameters, $\thetav$, and the parameters of the recurrent neural transformation, $\phiv(\cdot)$, denoted $W$ and further referred as the \emph{transformation hyperparameters}.
Our goal is to optimize $\lik$ with respect to both $\thetav$ and $W$.

The derivative of the NLML objective with respect to $\thetav$ is standard and takes the following form~\citep{williams2006gaussian}:
\begin{equation}
    \frac{\partial\lik}{\partial \thetav} = \frac{1}{2} \tr{\left[K^{-1} \yv\yv^\top K^{-1} - K^{-1}\right]\frac{\partial K}{\partial \thetav}},
\end{equation}
where $\partial K / \partial \thetav$ is depends on the kernel function, $k(\cdot, \cdot)$, and usually has an analytic form.
The derivative with respect to the $l$-th transformation hyperparameter, $W_l$, is as follows:\footnote{Step-by-step derivations are given in Appendix~\ref{sec:grad-comp}.}
\begin{equation}
\label{eq:dlik_dw}
\frac{\partial\lik}{\partial W_l} =\frac{1}{2} \sum_{i,j} \left(K^{-1} \yv\yv^\top K^{-1} - K^{-1}\right)_{ij}
\left\{\left(\frac{\partial k(\hv_i, \hv_j)}{\partial \hv_i}\right)^\top \frac{\partial \hv_i}{\partial W_l} + \left(\frac{\partial k(\hv_i, \hv_j)}{\partial \hv_j}\right)^\top \frac{\partial \hv_j}{\partial W_l}\right\},
\end{equation}
where $\hv_i = \phiv(\overline{\xv}_i)$ corresponds to the latent representation of the the $i$-th data instance.
Once the derivatives are computed, the model can be trained with any first-order or quasi-Newton optimization routine.
However, application of the stochastic gradient method---the \emph{de facto} standard optimization routine for deep recurrent networks---is not straightforward:
neither the objective, nor its derivatives factorize over the data\footnote{Cannot be represented as sums of independent functions of each data point.} due to the kernel matrix inverses, and hence convergence is not guaranteed.

\textbf{Semi-stochastic alternating gradient descent.}
Observe that once the kernel matrix, $K$, is \emph{fixed}, the expression, $\left(K^{-1} \yv\yv^\top K^{-1} - K^{-1}\right)$, can be precomputed on the full data and fixed.
Subsequently, Eq.~\eqref{eq:dlik_dw} turns into a weighted sum of independent functions of each data point.
This observation suggests that, given a \emph{fixed kernel matrix}, one could compute a stochastic update for $W$ on a mini-batch of training points by only using the corresponding sub-matrix of $K$.
Hence, we propose to optimize GPs with recurrent kernels in a semi-stochastic fashion, alternating between updating the kernel hyperparameters, $\thetav$, on the full data first, and then updating the weights of the recurrent network, $W$, using stochastic steps.
The procedure is given in Algorithm~\ref{alg:semi-stoch-alt-grad}.

Semi-stochastic alternating gradient descent is a special case of block-stochastic gradient iteration~\citep{xu2015block}.
While the latter splits the variables into arbitrary blocks and applies Gauss–Seidel type stochastic gradient updates to each of them, our procedure alternates between applying deterministic updates to $\thetav$ and stochastic updates to $W$ of the form $\thetav^{(t+1)} \leftarrow \thetav^{(t)} + \lambda_{\thetav}^{(t)} \bg_{\thetav}^{(t)}$ and $W^{(t+1)} \leftarrow W^{(t)} + \lambda_{W}^{(t)} \bg_{W}^{(t)}$\footnote{In principle, stochastic updates of $\thetav$ are also possible. As we will see next, we choose in practice to keep the kernel matrix fixed while performing stochastic updates. Due to sensitivity of the kernel to even small changes in $\thetav$, convergence of the fully stochastic scheme is fragile.}.
The corresponding Algorithm~\ref{alg:semi-stoch-alt-grad} is provably convergent for convex and non-convex problems under certain conditions.
The following theorem adapts results of~\citet{xu2015block} to our optimization scheme.

%!TEX root = ../16-498.tex

\begin{figure}[t!]
\begin{minipage}[t]{0.48\textwidth}
\vspace{0pt}
\begin{algorithm}[H]
    \caption{Semi-stochastic alternating gradient descent.}
    \label{alg:semi-stoch-alt-grad}
    \begin{algorithmic}[1]
        \INPUT
            Data -- $(X, \yv)$,
            kernel -- $k_{\thetav}(\cdot, \cdot)$, \\
            recurrent transformation -- $\phi_{\wv}(\cdot)$.
        \STATE Initialize $\thetav$ and $\wv$; compute initial $K$.
        \REPEAT
            \FORALL{mini-batches $X_b$ in $X$}
                \STATE  \textcolor{dark-red}{$\thetav \leftarrow \thetav + \mathrm{\texttt{update}}_{\thetav}(X, \thetav, K)$.} and \\
                $\wv \leftarrow \wv + \mathrm{\texttt{update}}_{\wv}(X_b, \wv, K)$.
                \STATE \textcolor{dark-red}{Update the kernel matrix, $K$.}
            \ENDFOR
        \UNTIL{Convergence}
        \OUTPUT Optimal $\thetav^*$ and $\wv^*$
    \end{algorithmic}
\end{algorithm}
\end{minipage}
\hfill
\begin{minipage}[t]{0.48\textwidth}
\vspace{0pt}
\begin{algorithm}[H]
    \caption{Semi-stochastic asynchronous gradient descent.}
    \label{alg:semi-stoch-grad-stale-kernel}
    \begin{algorithmic}[1]
        \INPUT
            Data -- $(X, \yv)$,
            kernel -- $k_{\thetav}(\cdot, \cdot)$, \\
            recurrent transformation -- $\phi_{\wv}(\cdot)$.
        \STATE Initialize $\thetav$ and $\wv$; compute initial $K$.
        \REPEAT
            \STATE \textcolor{dark-green}{$\thetav \leftarrow \thetav + \mathrm{\texttt{update}}_{\thetav}(X, \wv, \textcolor{dark-green}{K})$.}
            \FORALL{mini-batches $X_b$ in $X$}
                \STATE $\wv \leftarrow \wv + \mathrm{\texttt{update}}_{\wv}(X_b, \wv, \textcolor{dark-green}{K^{\mathrm{stale}}})$.
            \ENDFOR
            \STATE \textcolor{dark-green}{Update the kernel matrix, $K$.}
        \UNTIL{Convergence}
        \OUTPUT Optimal $\thetav^*$ and $\wv^*$
    \end{algorithmic}
\end{algorithm}
\end{minipage}
\end{figure}

\begin{theorem}[informal]\label{thm:theorem1}
Semi-stochastic alternating gradient descent converges to a fixed point when the learning rate, $\lambda_t$, decays as $\Theta(1/t^\frac{1 + \delta}{2})$ for any $\delta \in (0, 1]$.
\end{theorem}

Applying alternating gradient to our case has a catch: the kernel matrix (and its inverse) has to be updated each time $W$ and $\thetav$ are changed, \emph{i.e.}, on every mini-batch iteration (marked red in Algorithm~\ref{alg:semi-stoch-alt-grad}).
Computationally, this updating strategy defeats the purpose of stochastic gradients because we have to use the entire data on each step.
To deal with the issue of computational efficiency, we use ideas from asynchronous optimization.

\textbf{Asynchronous techniques.}
One of the recent trends in parallel and distributed optimization is applying updates in an asynchronous fashion~\citep{agarwal2011distributed}.
Such strategies naturally require some tolerance to delays in parameter updates~\citep{langford2009slow}.
In our case, we modify Algorithm~\ref{alg:semi-stoch-alt-grad} to allow \emph{delayed kernel matrix updates}.

The key observation is very intuitive: when the stochastic updates of $W$ are small enough, $K$ does not change much between mini-batches, and hence we can perform multiple stochastic steps for $W$ before re-computing the kernel matrix, $K$, and still converge.
For example, $K$ may be updated once at the end of each pass through the entire data (see Algorithm~\ref{alg:semi-stoch-grad-stale-kernel}).
To ensure convergence of the algorithm, it is important to strike the balance between (a) the learning rate for $W$ and (b) the frequency of the kernel matrix updates.
The following theorem provides convergence results under certain conditions.

\begin{theorem}[informal]\label{thm:theorem2}
Semi-stochastic gradient descent with $\tau$-delayed kernel updates converges to a fixed point when the learning rate, $\lambda_t$, decays as $\Theta(1/\tau t^\frac{1 + \delta}{2})$ for any $\delta \in (0, 1]$.
\end{theorem}

\noindent Formal statements, conditions, and proofs for Theorems~\ref{thm:theorem1}~and~\ref{thm:theorem2} are given in Appendix~\ref{sec:convergence-alg-2}.

\textbf{Why stochastic optimization?}
GPs with recurrent kernels can be also trained with full-batch gradient descent, as proposed by~\citet{wilson2016dkl}.
However, stochastic gradient methods have been proved to attain better generalization~\citep{hardt2015train} and often demonstrate superior performance in deep and recurrent architectures~\citep{wilson2003general}.
Moreover, stochastic methods are `online', i.e., they update model parameters based on subsets of an incoming data stream, and hence can scale to very large datasets.
In our experiments, we demonstrate that GPs with recurrent kernels trained with Algorithm~2 converge faster (i.e., require fewer passes through the data) and attain better performance than if trained with full-batch techniques.

\textbf{Stochastic variational inference.}
Stochastic variational inference (SVI) in Gaussian processes~\citep{hensman2013gaussian} is another viable approach to enabling stochastic optimization for GPs with recurrent kernels.
Such method would optimize a variational lower bound on the original objective that factorizes over the data by construction.  Recently, \citet{wilson2016stochastic} developed such a stochastic variational approach in the context of deep kernel learning.
Note that unlike all previous existing work, our proposed approach does not require a variational approximation to the marginal likelihood to perform mini-batch training of Gaussian processes.

\subsection{Scalability}\label{sec:scalability}
Learning and inference with Gaussian processes requires solving a linear system involving an $n \times n$ kernel matrix, $K^{-1} {y}$, and computing a log determinant over $K$.
These operations typically require $\Ocal(n^3)$ computations for $n$ training data points, and $\Ocal(n^2)$ storage.
In our approach, scalability is achieved through semi-stochastic training and structure-exploiting inference.
In particular, asynchronous semi-stochastic gradient descent reduces both the total number of passes through the data required for the model to converge and the number of calls to the linear system solver; exploiting the structure of the kernels significantly reduces the time and memory complexities of the linear algebraic operations.

More precisely, we replace all instances of the covariance matrix $K_\yv$ with $W K_{U,U} W^{\top}$, where $W$ is a sparse interpolation matrix, and $K_{U,U}$ is the covariance matrix evaluated over $m$ latent inducing points, which decomposes into a Kronecker product of circulant matrices~\citep{wilson2015kissgp,wilson2015msgp}.
This construction makes inference and learning scale as $\Ocal(n)$ and test predictions be $\Ocal(1)$, while preserving model structure.
For the sake of completeness, we provide an overview of the underlying algebraic machinery in Appendix~\ref{sec:msgp}.

At a high level, because $W$ is sparse and $K_{U,U}$ is structured it is possible to take extremely fast matrix vector multiplications (MVMs) with the approximate covariance matrix $K_{X,X}$.
One can then use methods such as linear conjugate gradients, which only use MVMs, to efficiently solve linear systems.
MVM or scaled eigenvalue approaches~\citep{wilson2015kissgp,wilson2015msgp} can also be used to efficiently compute the log determinant and its derivatives.
Kernel interpolation~\citep{wilson2015msgp} also enables fast predictions, as we describe further in the Appendix.

%!TEX root = ../16-498.tex

%!TEX root = ../16-498.tex

\begin{table*}[t!]
\centering
\caption{\small
Statistics for the data used in experiments.
SNR was determined by assuming a certain degree of smoothness of the signal, fitting kernel ridge regression with RBF kernel to predict the targets from the input time series, and regarding the residuals as the noise.
Tasks with low average correlation between inputs and targets and lower SNR are harder prediction problems.}
\label{tab:datasets}
\vspace{-1ex}
\small
\begin{tabular}{llrrrrr}
\toprule
\textbf{Dataset} & \textbf{Task} & \textbf{\# time steps} & \textbf{\# dim} & \textbf{\# outputs} & \textbf{Abs. corr.} & \textbf{SNR} \\
\midrule
\textit{Drives}     & \multirow{2}{*}{system ident.}
                    & 500       & 1     & 1     & 0.7994    & 25.72 \\
\textit{Actuator}   &
                    & 1,024     & 1     & 1     & 0.0938    & 12.47 \\
\midrule
\multirow{2}{*}{\textit{GEF}}
                    & power load
                    & 38,064    & 11    & 1     & 0.5147    & 89.93 \\
                    & wind power
                    & 130,963   & 16    & 1     & 0.1731    & 4.06 \\
\midrule
\multirow{4}{*}{\textit{Car}}
                    & speed
                    & \multirow{4}{*}{932,939}
                                & 6     & 1     & 0.1196    & 159.33 \\
                    & gyro yaw
                    &           & 6     & 1     & 0.0764    & 3.19 \\
                    & lanes
                    &           & 26    & 16    & 0.0816    & --- \\
                    & lead vehicle
                    &           & 9     & 2     & 0.1099    & --- \\
\bottomrule
\end{tabular}
\vspace{-1ex}
\end{table*}

\section{Experiments}\label{sec:experiments}

We compare the proposed Gaussian processes with recurrent kernels based on RNN and LSTM architectures (GP-RNN/LSTM) with a number of baselines on datasets of various complexity and ranging in size from hundreds to almost a million of time points.
For the datasets with more than a few thousand points, we use a massively scalable version of GPs (see Section~\ref{sec:scalability}) and demonstrate its scalability during inference and learning.
We carry out a number of experiments that help to gain empirical insights about the convergence properties of the proposed optimization procedure with delayed kernel updates.
Additionally, we analyze the regularization properties of GP-RNN/LSTM and compare them with other techniques, such as dropout.
Finally, we apply the model to the problem of lane estimation and lead vehicle position prediction, both critical in autonomous driving applications.

\subsection{Data and the setup}

Below, we describe each dataset we used in our experiments and the associated prediction tasks.
The essential statistics for the datasets are summarized in Table~\ref{tab:datasets}.

\textbf{System identification.}
In the first set of experiments, we used publicly available nonlinear system identification datasets: \textit{Actuator}\footnote{\url{http://www.iau.dtu.dk/nnbook/systems.html}} \citep{sjoberg1995nonlinear} and \textit{Drives}\footnote{\url{http://www.it.uu.se/research/publications/reports/2010-020/NonlinearData.zip}.}~\citep{wigren2010input}.
Both datasets had one dimensional input and output time series.
\textit{Actuator} had the size of the valve opening as the input and the resulting change in oil pressure as the output.
\textit{Drives} was from a system with motors that drive a pulley using a flexible belt; the input was the sum of voltages applied to the motors and the output was the speed of the belt.

\textbf{Smart grid data\footnote{The smart grid data were taken from Global Energy Forecasting Kaggle competitions organized in 2012.}.}
We considered the problem of forecasting for the smart grid that consisted of two tasks (Figure~\ref{fig:gef}).
The first task was to predict power load from the historical temperature data.
The data had 11 input time series coming from hourly measurements of temperature on 11 zones and an output time series that represented the cumulative hourly power load on a U.S. utility.
The second task was to predict power generated by wind farms from the wind forecasts.
The data consisted of 4 different hourly forecasts of the wind and hourly values of the generated power by a wind farm.
Each wind forecast was a 4-element vector that corresponded to \textit{zonal component}, \textit{meridional component}, \textit{wind speed} and \textit{wind angle}.
In our experiments, we concatenated the 4 different 4-element forecasts, which resulted in a 16-dimensional input time series.

%!TEX root = ../16-498.tex

\begin{figure}[t!]
\centering
\includegraphics[width=0.485\textwidth]{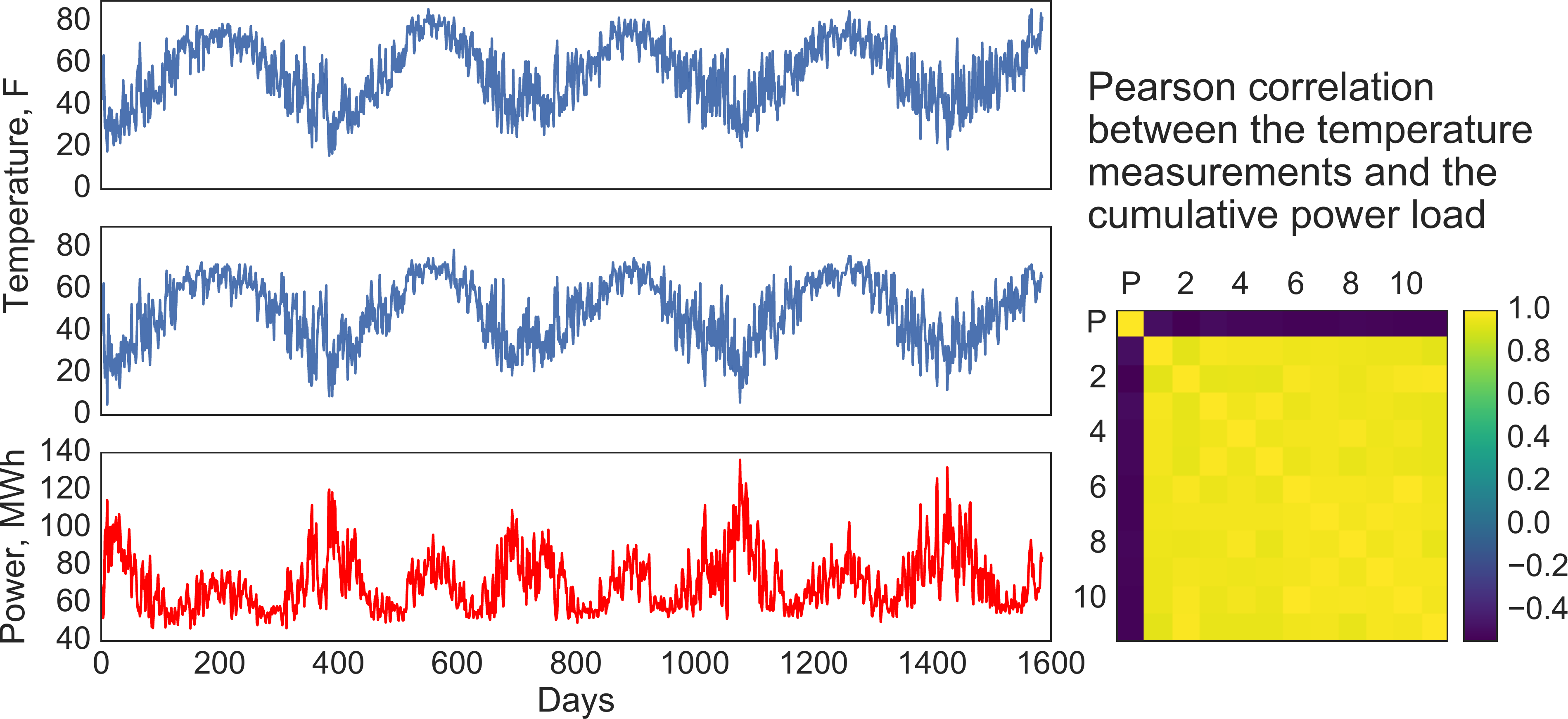}
\hfill
\includegraphics[width=0.485\textwidth]{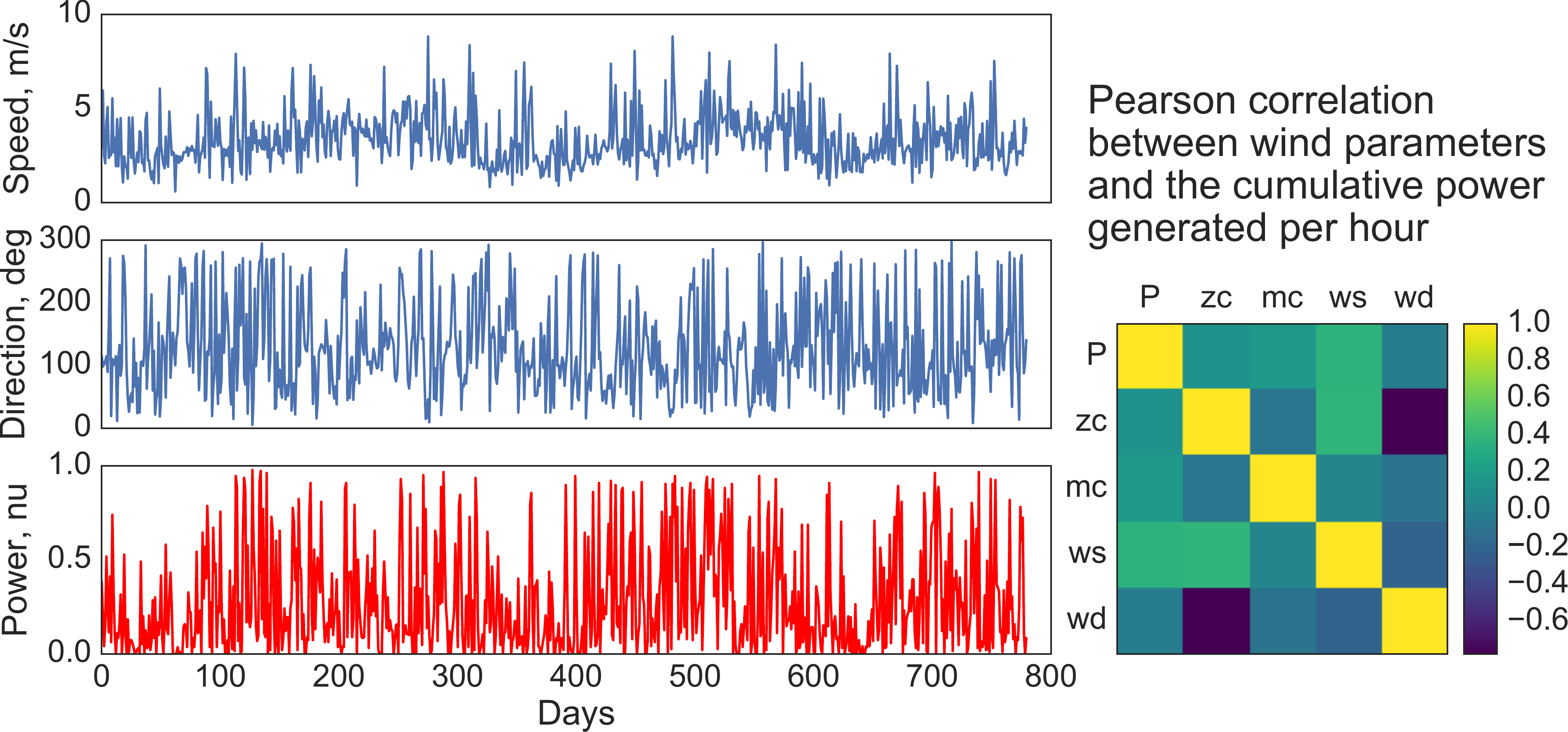}
\caption{
\textbf{Left:} Visualization of the \textit{GEF-power} time series for two zones and the cumulative load with the time resolution of 1 day.
Cumulative power load is generally negatively correlated with the temperature measurements on all the zones.
\textbf{Right:} Visualization of the \textit{GEF-wind} time series with the time resolution of 1 day.}
\label{fig:gef}
\end{figure}

\textbf{Self-driving car dataset\footnote{The dataset is proprietary.
It was released in part for public use under the Creative Commons Attribution 3.0 license: \url{http://archive.org/details/comma-dataset}.
More about the self-driving car: \url{http://www.bloomberg.com/features/2015-george-hotz-self-driving-car/}.
}.}
One of the main target applications of the proposed model is prediction for autonomous driving.
We considered a large dataset coming from sensors of a self-driving car that was recorded on two car trips with discretization of $10\, ms$.
The data featured two sets of GPS ECEF locations, ECEF velocities, measurements from a fiber-optic gyro compass, LIDAR, and a few more time series from a variety of IMU sensors.
Additionally, locations of the left and right lanes were extracted from a video stream for each time step as well as the position of the lead vehicle from the LIDAR measurements.
We considered the data from the first trip for training and from the second trip for validation and testing.
A visualization of the car routes with 25 second discretization in the ENU coordinates are given in Figure~\ref{fig:car-data}.
We consider four tasks, the first two of which are more of proof-of-concept type variety, while the final two are fundamental to good performance for a self-driving car:
\begin{enumerate}[itemsep=1pt, topsep=1pt]
\item Speed prediction from noisy GPS velocity estimates and gyroscopic inputs.
\item Prediction of the angular acceleration of the car from the estimates of its speed and steering angle.
\item Point-wise prediction of the lanes from the estimates at the previous time steps, and estimates of speed, gyroscopic and compass measurements.
\item Prediction of the lead vehicle location from its location at the previous time steps, and estimates of speed, gyroscopic and compass measurements.
\end{enumerate}
We provide more specific details on the smart grid data and self-driving data in Appendix~\ref{sec:data-details}.\\[-2ex]

%!TEX root = ../16-498.tex

\begin{figure}[t!]
\begin{adjustbox}{valign=t}
\begin{minipage}{0.65\textwidth}
\centering
\includegraphics[width=\textwidth]{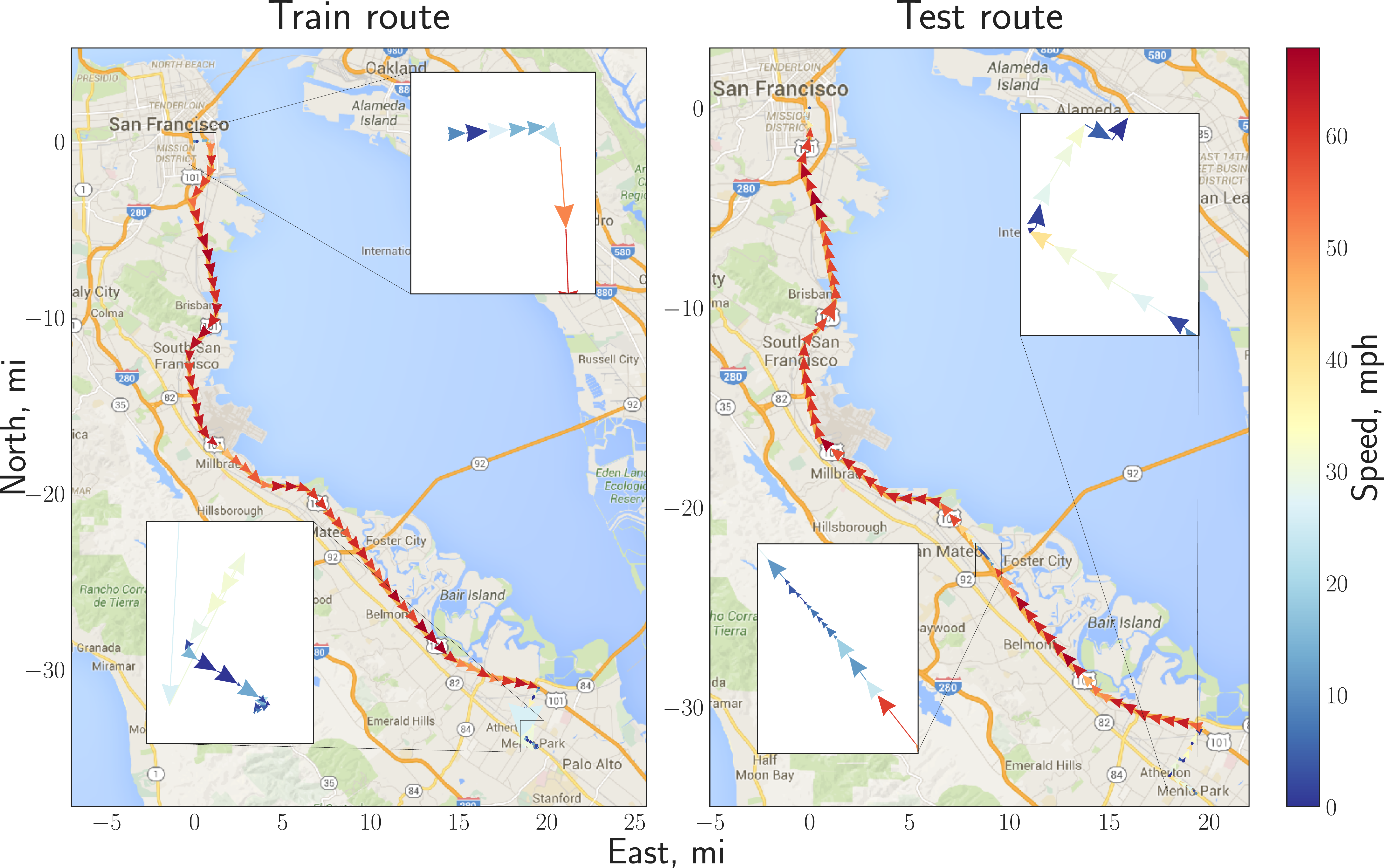}
\end{minipage}
\end{adjustbox}
~
\begin{adjustbox}{valign=t}
\begin{minipage}{0.34\textwidth}
\centering
\bgroup
\def\arraystretch{1.3}%  1 is the default, change whatever you need
\scriptsize
\begin{tabular}{lrr}
\toprule
& \textbf{Train} & \textbf{Test} \\
\midrule
Time & 42.8 min & 46.5 min \\
\midrule
\multicolumn{3}{l}{\textbf{Speed of the self-driving car}} \\
Min & 0.0 mph & 0.0 mph \\
Max & 80.3 mph & 70.1 mph \\
Average & 42.4 mph & 38.4 mph \\
Median & 58.9 mph & 44.8 mph \\
\midrule
\multicolumn{3}{l}{\textbf{Distance to the lead vehicle}} \\
Min & 23.7 m & 29.7 m \\
Max & 178.7 m & 184.7 m \\
Average & 85.4 m & 72.6 m \\
Median & 54.9 m & 53.0 m \\
\bottomrule
\end{tabular}
\egroup
\end{minipage}
\end{adjustbox}
\captionlistentry[table]{}
\captionsetup{labelformat=andtable}
\caption{
\textbf{Left:} Train and test routes of the self-driving car in the ENU coordinates with the origin at the starting location.
Arrows point in the direction of motion; color encodes the speed.
Insets zoom selected regions of the routes.
Best viewed in color.
\textbf{Right:} Summary of the data collected on the train and test routes.
}\label{fig:car-data}
\vspace{-2ex}
\end{figure}

\textbf{Models and metrics.}
We used a number of classical baselines: NARX~\citep{lin1996learning}, GP-NARX models~\citep{kocijan2005dynamic}, and classical RNN and LSTM architectures.
The kernels of our models, GP-NARX/RNN/LSTM, used the ARD base kernel function and were structured by the corresponding baselines.\footnote{GP-NARX is a special instance of our more general framework and we trained it using the proposed semi-stochastic algorithm.}
As the primary metric, we used root mean squared error (RMSE) on a held out set and additionally negative log marginal likelihood (NLML) on the training set for the GP-based models.

We train all the models to perform one-step-ahead prediction in an autoregressive setting, where targets at the future time steps are predicted from the input and target values at a fixed number of past time steps.
For the system identification task, we additionally consider the non-autoregressive scenario (i.e., mapping only input sequences to the future targets), where we are performing prediction in the \emph{free simulation mode}, and included recurrent Gaussian processes~\citep{mattos2015recurrent} in our comparison.
In this case, none of the future targets are available and the models have to re-use their own past predictions to produce future forecasts).

\textbf{A note on implementation.}
Recurrent parts of each model were implemented using \textit{Keras}\footnote{\url{http://www.keras.io}} library.
We extended \textit{Keras} with the Gaussian process layer and developed a backed engine based on the \textit{GPML} library\footnote{\url{http://www.gaussianprocess.org/gpml/code/matlab/doc/}}.
Our approach allows us to take full advantage of the functionality available in \textit{Keras} and \textit{GPML}, \emph{e.g.}, use automatic differentiation for the recurrent part of the model.
Our code is available at \url{http://github.com/alshedivat/kgp/}.

\subsection{Analysis}

This section discusses quantitative and qualitative experimental results.
We only briefly introduce the model architectures and the training schemes used in each of the experiments.
We provide a comprehensive summary of these details in Appendix~\ref{sec:architecture-details}.

\subsubsection{Convergence of the optimization}

To address the question of whether stochastic optimization of recurrent kernels is necessary and to assess the behavior of the proposed optimization scheme with delayed kernel updates, we conducted a number of experiments on the Actuator dataset (Figure~\ref{fig:optimization}).

%!TEX root = ../16-498.tex

\begin{figure}[t]
\centering
\includegraphics[width=\textwidth]{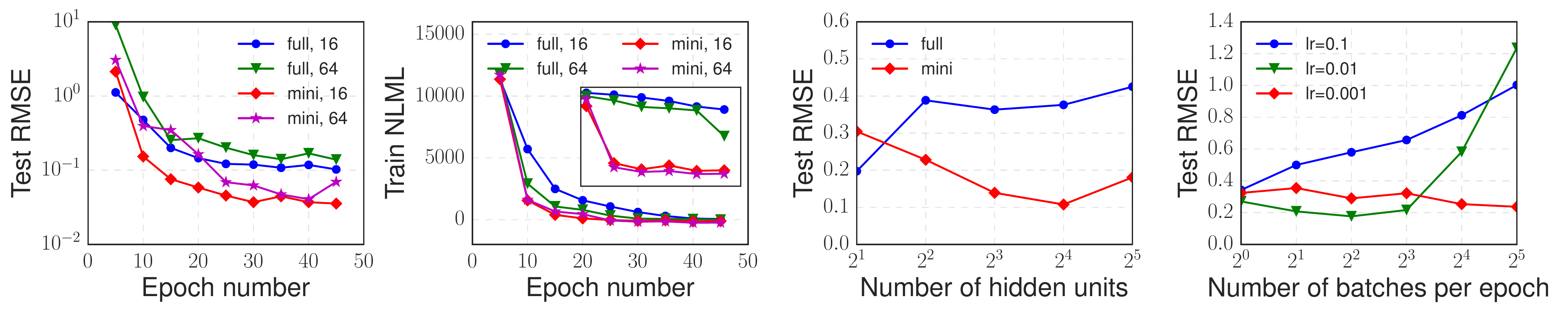}
\caption{\small
\emph{Two charts on the left:}
Convergence of the optimization in terms of RMSE on test and NLML on train.
The inset zooms the region of the plot right beneath it using log scale for the vertical axis.
\emph{full} and \emph{mini} denote full-batch and mini-batch optimization procedures, respectively, while 16 and 64 refer to models with the respective number of units per hidden layer.
\emph{Two charts on the right:}
Test RMSE for a given architecture trained with a specified method and/or learning rate.
}
\label{fig:optimization}
\end{figure}

First, we constructed two GP-LSTM models with 1 recurrent hidden layer and 16 or 64 hidden units and trained them with (non-stochastic) full-batch iterative procedure (similar to the proposal of~\citet{wilson2016dkl}) and with our semi-stochastic optimizer with delayed kernel updates (Algorithm~\ref{alg:semi-stoch-grad-stale-kernel}).
The convergence results are given on the first two charts.
Both in terms of the error on a held out set and the NLML on the training set, the models trained with mini-batches converged faster and demonstrated better final performance.

Next, we compared the two optimization schemes on the same GP-LSTM architecture with different sizes of the hidden layer ranging from 2 to 32.
It is clear from the third chart that, even though full-batch approach seemed to find a better optimum when the number of hidden units was small, the stochastic approach was clearly superior for larger hidden layers.

Finally, we compared the behavior of Algorithm~\ref{alg:semi-stoch-grad-stale-kernel} with different number of mini-batches used for each epoch (equivalently, the number of steps between the kernel matrix updates) and different learning rates.
The results are give on the last chart.
As expected, there is a fine balance between the number of mini-batches and the learning rate: if the number of mini-batches is large (\emph{i.e.}, the delay between the kernel updates becomes too long) while the learning rate is high enough, optimization does not converge; at the same time, an appropriate combination of the learning rate and the mini-batch size leads better generalization than the default batch approach of~\cite{wilson2016dkl}.

\subsubsection{Regression, Auto-regression, and Free Simulation}
In this set of experiments, our main goal is to provide a comparison between three different modes of one-step-ahead prediction, referred to as (i) regression, (ii) autoregression, and (iii) free simulation, and compare performance of our models with RGP---a classical RNN with every parametric layer substituted with a Gaussian process~\citep{mattos2015recurrent}---on the Actuator and Drives datasets.
The difference between the prediction modes consists in whether and how the information about the past targets is used.
In the regression scenario, inputs and targets are separate time series and the model learns to map input values at a number of past time points to a target value at a future point in time.
Autoregression, additionally, uses the \emph{true} past target values as inputs; in the free simulation mode, the model learns to map past inputs and its own past predictions to a future target.

In the experiments in autoregression and free simulation modes, we used short time lags, $L = 10$, as suggested by~\citet{mattos2015recurrent}.
In the regression mode, since the model does not build the recurrent relationships based on the information about the targets (or their estimates), it generally requires larger time lags that can capture the state of the dynamics.
Hence we increased the time lag to 32 in the regression mode.
More details are given in Appendix~\ref{sec:architecture-details}.

%!TEX root = ../16-498.tex

\begin{table}[t!]
\centering
\caption{\small
Average performance of the models in terms of RMSE on the system identification tasks.
The averages were computed over 5 runs; the standard deviation is given in the parenthesis.
Results for the RGP model are as reported by \citet{mattos2015recurrent}, available only for the free simulation.}
\label{tab:performance-sysid}
\scriptsize
\begin{tabular}{l|r|r|r|rrr}
\toprule
\multirow{2}{*}{} &
\multicolumn{3}{c|}{\textbf{Drives}} &
\multicolumn{3}{c}{\textbf{Actuator}} \\
& regression & auto-regression & free simulation & regression & auto-regression & free simulation \\
\midrule
\textbf{NARX}       & 0.33 (0.02) & 0.19 (0.03) & 0.38 (0.03) & 0.49 (0.05) & 0.18 (0.01) & 0.57 (0.04) \\
\textbf{RNN}        & 0.53 (0.02) & 0.17 (0.04) & 0.56 (0.03) & 0.56 (0.03) & 0.17 (0.01) & 0.68 (0.05) \\
\textbf{LSTM}       & 0.29 (0.02) & \textbf{0.14} (0.02) & 0.40 (0.02) & 0.40 (0.03) & 0.19 (0.01) & 0.44 (0.03) \\
\midrule
\textbf{GP-NARX}    & 0.28 (0.02) & 0.16 (0.04) & 0.28 (0.02) & 0.46 (0.03) & \textbf{0.14} (0.01) & 0.63 (0.04) \\
\textbf{GP-RNN}     & 0.37 (0.04) & 0.16 (0.03) & 0.45 (0.03) & 0.49 (0.02) & \textbf{0.15} (0.01) & 0.55 (0.04) \\
\textbf{GP-LSTM}    & \textbf{0.25} (0.02) & \textbf{0.13} (0.02) & 0.32 (0.03) & 0.36 (0.01) & \textbf{0.14} (0.01) & 0.43 (0.03) \\
\midrule
\textbf{RGP}        & --- & --- & \textbf{0.249} & --- & --- & \textbf{0.368} \\
\bottomrule
\end{tabular}
\end{table}

We present the results in Table~\ref{tab:performance-sysid}.
We note that GP-based architectures consistently yielded improved predictive performance compared to their vanilla deep learning counterparts on both of the datasets, in each mode.
Given the small size of the datasets, we attribute such behavior to better regularization properties of the negative log marginal likelihood loss function.
We also found out that when GP-based models were initialized with weights of pre-trained neural networks, they tended to overfit and give overly confident predictions on these tasks.
The best performance was achieved when the models were trained from a random initialization~\citep[contrary to the findings of][]{wilson2016dkl}.
In free simulation mode RGP performs best of the compared models.
This result is expected---RGP was particularly designed to represent and propagate uncertainty through a recurrent process.  Our framework focuses on using recurrence to build expressive kernels for regression on sequences.

The suitability of each prediction mode depends on the task at hand.
In many applications where the future targets become readily available as the time passes (e.g., power estimation or stock market prediction), the autoregression mode is preferable.
We particularly consider autoregressive prediction in the further experiments.

\subsubsection{Prediction for smart grid and self-driving car applications}
For both smart grid prediction tasks we used LSTM and GP-LSTM models with 48 hour time lags and were predicting the target values one hour ahead.
LSTM and GP-LSTM were trained with one or two layers and 32 to 256 hidden units.
The best models were selected on 25\% of the training data used for validation.
For autonomous driving prediction tasks, we used the same architectures but with 128 time steps of lag (1.28~$s$).
All models were regularized with dropout~\citep{srivastava2014dropout,gal2016theoretically}.
On both GEF and self-driving car datasets, we used the scalable version of Gaussian process (MSGP)~\citep{wilson2015msgp}.
Given the scale of the data and the challenge of nonlinear optimization of the recurrent models, we initialized the recurrent parts of GP-RNN and GP-LSTM with pre-trained weights of the corresponding neural networks.
Fine-tuning of the models was performed with Algorithm~\ref{alg:semi-stoch-grad-stale-kernel}.
The quantitative results are provided in Table~\ref{tab:performance-gef-car} and demonstrate that GPs with recurrent kernels attain the state-of-the-art performance.

%!TEX root = ../16-498.tex

\begin{table}[t!]
\centering
\caption{\small
Average performance of the best models in terms of RMSE on the GEF and Car tasks.
The averages were computed over 5 runs; the standard deviation is given in the parenthesis.}
\label{tab:performance-gef-car}
\scriptsize
\begin{tabular}{l|rr|rrrr}
\toprule
\multirow{2}{*}{} & \multicolumn{2}{c|}{\textbf{GEF}} & \multicolumn{4}{c}{\textbf{Car}} \\
& power load & wind power & speed & gyro yaw & lanes & lead vehicle \\
\midrule
\textbf{NARX}       & 0.54 (0.02) & 0.84 (0.01) & 0.114 (0.010) & 0.19 (0.01) & 0.13 (0.01) & 0.41 (0.02) \\
\textbf{RNN}        & 0.61 (0.02) & 0.81 (0.01) & 0.152 (0.012) & 0.22 (0.01) & 0.33 (0.02) & 0.44 (0.03) \\
\textbf{LSTM}       & 0.45 (0.01) & \textbf{0.77} (0.01) & 0.027 (0.008) & 0.13 (0.01) & 0.08 (0.01) & 0.40 (0.01) \\
\midrule
\textbf{GP-NARX}    & 0.78 (0.03) & 0.83 (0.02) & 0.125 (0.015) & 0.23 (0.02) & 0.10 (0.01) & 0.34 (0.02) \\
\textbf{GP-RNN}     & 0.24 (0.02) & 0.79 (0.01) & 0.089 (0.013) & 0.24 (0.01) & 0.46 (0.08) & 0.41 (0.02) \\
\textbf{GP-LSTM}    & \textbf{0.17} (0.02) & \textbf{0.76} (0.01) &  \textbf{0.019} (0.006) & \textbf{0.08} (0.01) & \textbf{0.06} (0.01) & \textbf{0.32} (0.02) \\
\bottomrule
\end{tabular}
\end{table}

%!TEX root = ../16-498.tex

\begin{figure}[t!]
\centering
\begin{subfigure}[b]{\textwidth}
    \includegraphics[width=\textwidth]{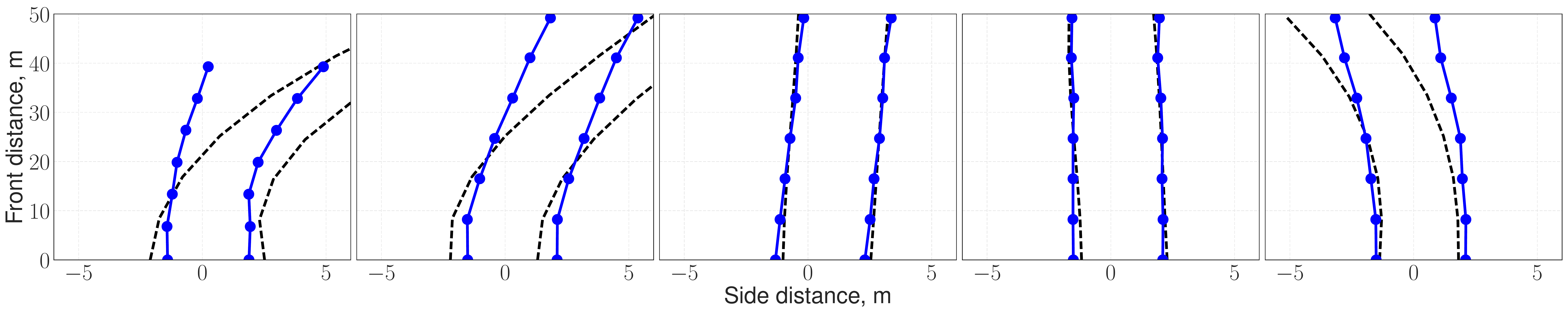}
    \includegraphics[width=\textwidth]{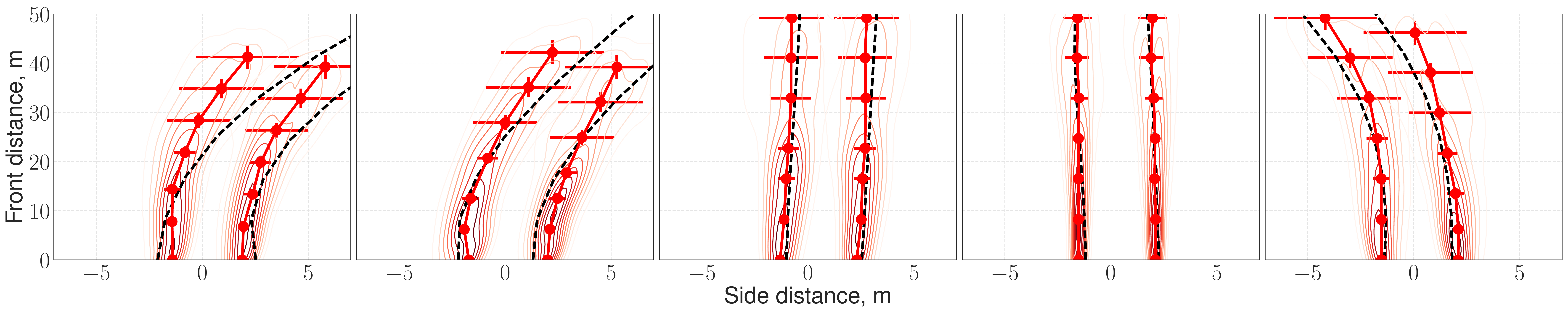}
    \caption{%
        Point-wise predictions of the lanes made by LSTM (upper) and by GP-LSTM (lower).
        Dashed lines correspond to the ground truth extracted from video sequences and used for training.
    }\label{fig:car-lanes}
\end{subfigure}
\begin{subfigure}[b]{\textwidth}
    \includegraphics[width=\textwidth]{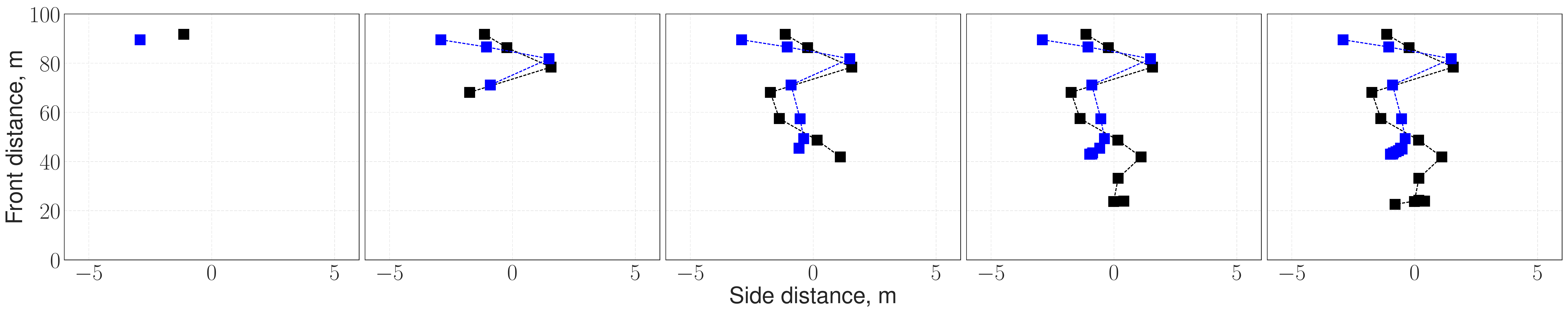}
    \includegraphics[width=\textwidth]{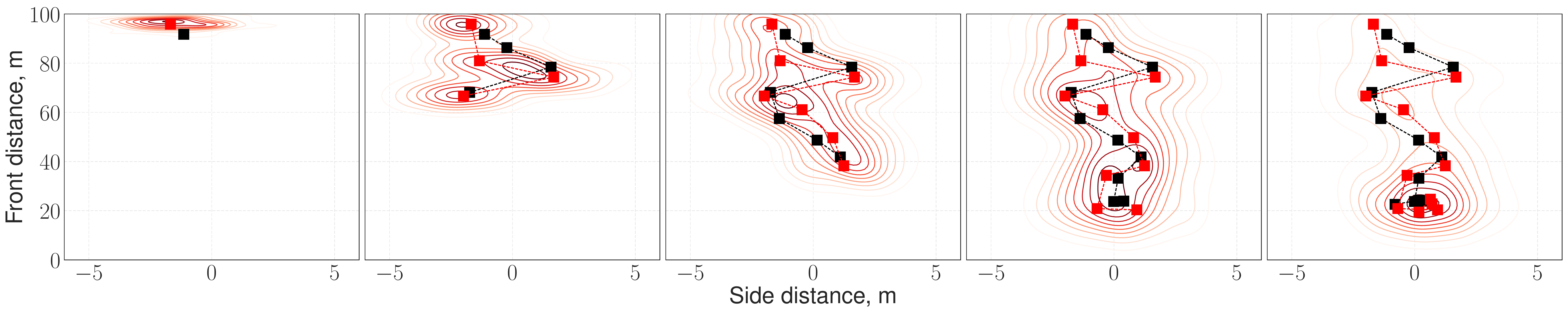}
    \caption{%
        LSTM (upper) and by GP-LSTM (lower) position predictions of the lead vehicle.
        Black markers and dashed lines are the ground truth; blue and red markers with solid lines correspond to predictions.
    }\label{fig:car-front-vehicle}
\end{subfigure}
\caption{%
Qualitative comparison of the LSTM and GP-LSTM predictions on self-driving tasks.
Predictive uncertainty of the GP-LSTM model is showed by contour plots and error-bars; the latter denote one standard deviation of the predictive distributions.
}\label{fig:car-qualitative}
\end{figure}

Additionally, we investigated convergence and regularization properties of LSTM and GP-LSTM models on the GEF-power dataset.
The first two charts of Figure~\ref{fig:performance} demonstrate that GP-based models are less prone to overfitting, even when the data is not enough.
The third panel shows that architectures with a particular number of hidden units per layer attain the best performance on the power prediction task.
An additional advantage of the GP-layers over the standard recurrent networks is that the best architecture could be identified based on the negative log likelihood of the model as shown on the last chart.

\newpage

Finally, Figure~\ref{fig:car-qualitative} qualitatively demonstrates the difference between the predictions given by LSTM vs. GP-LSTM on point-wise lane estimation (Figure~\ref{fig:car-lanes}) and the front vehicle tracking (Figure~\ref{fig:car-front-vehicle}) tasks.
We note that GP-LSTM not only provides a more robust fit, but also estimates the uncertainty of its predictions.
Such information can be further used in downstream prediction-based decision making, \emph{e.g.}, such as whether a self-driving car should slow down and switch to a more cautious driving style when the uncertainty is high.

%!TEX root = ../16-498.tex

\begin{figure}[t]
\centering
\includegraphics[width=\textwidth]{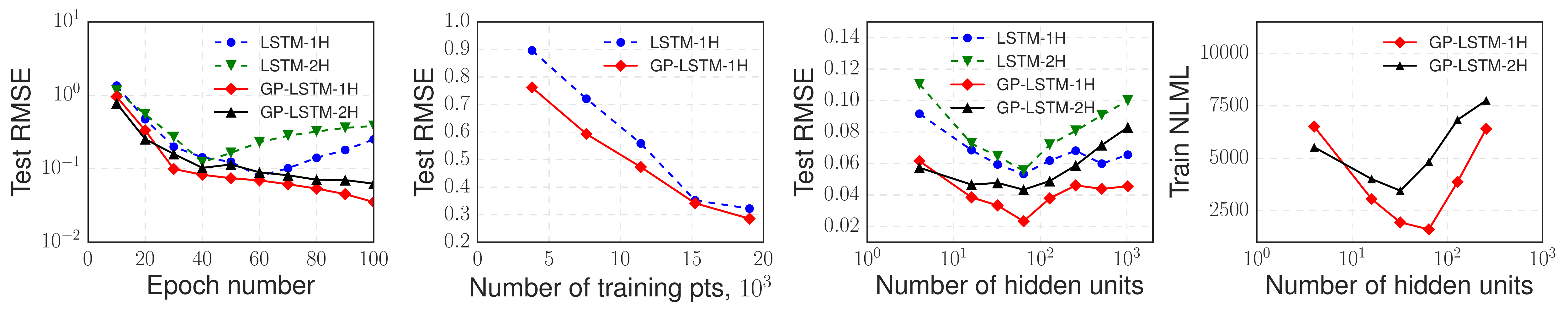}
\caption{\small
\emph{Left to right}: RMSE vs. the number of training points; RMSE vs. the number model parameters per layer; NLML vs. the number model parameters per layer for GP-based models.
All metrics are averages over 5 runs with different random initializations, computed on a held-out set.
}
\label{fig:performance}
\vspace{-1ex}
\end{figure}

%!TEX root = ../16-498.tex

\begin{figure}[t]
\centering
\includegraphics[width=\textwidth]{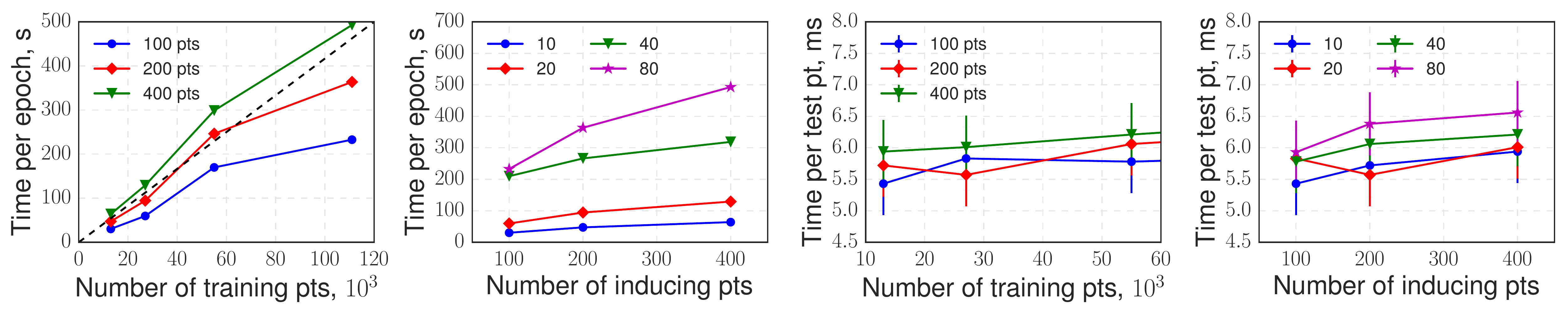}
\caption{\small
The charts demonstrate scalability of learning and inference of MSGP with an LSTM-based recurrent kernel.
Legends with points denote the number of inducing points used.
Legends with percentages denote the percentage of the training dataset used learning the model.
}
\label{fig:scalability}
\end{figure}

\subsubsection{Scalability of the model}
Following~\cite{wilson2015msgp}, we performed a generic scalability analysis of the MSGP-LSTM model on the car sensors data.
The LSTM architecture was the same as described in the previous section: it was transforming multi-dimensional sequences of inputs to a two-dimensional representation.
We trained the model for 10 epochs on 10\%, 20\%, 40\%, and 80\% of the training set with 100, 200, and 400 inducing points per dimension and measured the average training time per epoch and the average prediction time per testing point.
The measured time was the total time spent on both LSTM optimization and MSGP computations.
The results are presented in Figure~\ref{fig:scalability}.

The training time per epoch (one full pass through the entire training data) grows linearly with the number of training examples and depends linearly on the number of inducing points (Figure~\ref{fig:scalability}, two left charts).
Thus, given a fixed number of inducing points per dimension, the time complexity of MSGP-LSTM learning and inference procedures is linear in the number of training examples.
The prediction time per testing data point is virtually constant and does not depend on neither on the number of training points, nor on the number of inducing points (Figure~\ref{fig:scalability}, two right charts).

%!TEX root = ../16-498.tex

\section{Discussion}
We proposed a method for learning kernels with recurrent long short-term memory structure on sequences.
Gaussian processes with such kernels, termed the GP-LSTM, have the structure and learning biases of LSTMs, while retaining a probabilistic Bayesian nonparametric representation.
The GP-LSTM outperforms a range of alternatives on several sequence-to-reals regression tasks.
The GP-LSTM also works on data with low and high signal-to-noise ratios, and can be scaled to very large datasets, all with a straightforward, practical, and generally applicable model specification.
Moreover, the semi-stochastic scheme proposed in our paper is provably convergent and efficient in practical settings, in conjunction with structure exploiting algebra.
In short, the GP-LSTM provides a natural mechanism for Bayesian LSTMs, quantifying predictive uncertainty while harmonizing with the standard deep learning toolbox.  Predictive uncertainty is of high value in robotics applications, such as autonomous driving, and could also be applied to other areas such as financial modeling and computational biology.

There are several exciting directions for future research.
The GP-LSTM quantifies predictive uncertainty but does not model the propagation of uncertainty in the inputs through a recurrent structure.
Treating free simulation as a structured prediction problem and using online corrective algorithms, e.g., DAGGER~\citep{ross2011reduction}, are likely to improve performance of GP-LSTM in the free prediction mode.
This approach would not require explicitly modeling and propagating uncertainty through the recurrence and would maintain the high computational efficiency of our method.

Alternatively, it would be exciting to have a probabilistic treatment of all parameters of the GP-LSTM kernel, including all LSTM weights.
Such an extension could be combined with stochastic variational inference, to enable both classification and non-Gaussian likelihoods as in~\citet{wilson2016stochastic}, but also open the doors to stochastic gradient Hamiltonian Monte Carlo~\citep{chen2014stochastic} (SG-HMC) for efficient inference over kernel parameters.
Indeed, SG-HMC has recently been used for efficient inference over network parameters in the Bayesian GAN~\citep{bayesiangan}.
A Bayesian approach to marginalizing the weights of the GP-LSTM kernel would also provide a principled probabilistic mechanism for learning model hyperparameters.

One could relax several additional assumptions.
We modeled each output dimension with independent GPs that shared a recurrent transformation.
To capture the correlations between output dimensions, it would be promising to move to a multi-task formulation.
In the future, one could also learn the time horizon in the recurrent transformation, which could lead to major additional performance gains.

Finally, the semi-stochastic learning procedure naturally complements research in asynchronous optimization \citep[e.g.,][]{deisenroth2015distributed}.
In combination with stochastic variational inference, the semi-stochastic approach could be used for parallel kernel learning, side-stepping the independence assumptions in prior work.
We envision that such efforts for Gaussian processes will harmonize with current progress in Bayesian deep learning.

%!TEX root = ../16-498.tex

\section{Acknowledgements}
The authors thank Yifei Ma for helpful discussions and the anonymous reviewers for the valuable comments that helped to improve the paper.
This work was supported in part by NIH R01GM114311, AFRL/DARPA FA87501220324, and NSF IIS-1563887.

\clearpage
%!TEX root = ../16-498.tex

\appendix
\setcounter{equation}{0}
\renewcommand{\theequation}{\Alph{section}.\arabic{equation}}

\section{Massively scalable Gaussian processes}\label{sec:msgp}
Massively scalable Gaussian processes (MSGP)~\citep{wilson2015msgp} is a significant extension of the kernel interpolation framework originally proposed by~\citet{wilson2015kissgp}.
The core idea of the framework is to improve scalability of the inducing point methods~\citep{quinonero2005unifying} by (1) placing the virtual points on a regular grid, (2) exploiting the resulting Kronecker and Toeplitz structures of the relevant covariance matrices, and (3) do local cubic interpolation to go back to the kernel evaluated at the original points.
This combination of techniques brings the complexity down to $\Ocal(n)$ for training and $\Ocal(1)$ for each test prediction.
Below, we overview the methodology.
We remark that a major difference in philosophy between MSGP and many classical inducing point methods is that the points are \emph{selected and fixed} rather than optimized over.
This allows to use significantly more virtual points which typically results in a better approximation of the true kernel.

\subsection{Structured kernel interpolation}
Given a set of $m$ inducing points, the $n \times m$ cross-covariance matrix, $K_{X,U}$, between the training inputs, $X$, and the inducing points, $\Uv$, can be approximated as $\tilde K_{X,U} = W_X K_{U,U}$ using a (potentially sparse) $n \times m$ matrix of interpolation weights, $W_{X}$.
This allows to approximate $K_{\Xv,\Zv}$ for an arbitrary set of inputs $\Zv$ as $K_{\Xv,\Zv} \approx \tilde K_{X,U} W_{\Zv}^\top$.
For any given kernel function, $K$, and a set of inducing points, $\Uv$, \emph{structured kernel interpolation} (SKI) procedure~\citep{wilson2015kissgp} gives rise to the following approximate kernel:
\begin{equation}\label{eq:ski-kernel}
    K_\mathrm{SKI}(\xv, \zv) = W_{X} K_{U,U} W_{z}^\top,
\end{equation}
which allows to approximate $K_{X,X} \approx W_{X} K_{U,U} W_{X}^\top$.
\citet{wilson2015kissgp} note that standard inducing point approaches, such as subset of regression (SoR) or fully independent training conditional (FITC), can be reinterpreted from the SKI perspective.
Importantly, the efficiency of SKI-based MSGP methods comes from, first, a clever choice of a set of inducing points that allows to exploit algebraic structure of $K_{U,U}$, and second, from using very sparse local interpolation matrices.
In practice, local cubic interpolation is used~\citep{keys1981cubic}.

\subsection{Kernel approximations}
If inducing points, $U$, form a regularly spaced $P$-dimensional grid, and we use a stationary product kernel (e.g., the RBF kernel), then $K_{U,U}$ decomposes as a Kronecker product of Toeplitz matrices:
\begin{equation}
K_{U,U} = \Tv_1 \otimes \Tv_2 \otimes \cdots \otimes \Tv_P.
\end{equation}
The Kronecker structure allows to compute the eigendecomposition of $K_{U,U}$ by separately decomposing $\Tv_1, \dots, \Tv_P$, each of which is much smaller than $K_{U,U}$.
Further, to efficiently eigendecompose a Toeplitz matrix, it can be approximated by a circulant matrix\footnote{\citet{wilson2015msgp} explored 5 different approximation methods known in the numerical analysis literature.} which eigendecomposes by simply applying discrete Fourier transform (DFT) to its first column.
Therefore, an approximate eigendecomposition of each $\Tv_1, \dots, \Tv_P$ is computed via the fast Fourier transform (FFT) and requires only $\Ocal(m \log m)$ time.

\subsection{Structure exploiting inference}
To perform inference, we need to solve $(K_\mathrm{SKI} + \sigma^2 \Iv)^{-1}\yv$; kernel learning requires evaluating $\log\det(K_\mathrm{SKI} + \sigma^2 \Iv)$.
The first task can be accomplished by using an iterative scheme---linear conjugate gradients---which depends only on matrix vector multiplications with $(K_\mathrm{SKI} + \sigma^2 \Iv)$.
The second is done by exploiting the Kronecker and Toeplitz structure of $K_{U,U}$ for computing an approximate eigendecomposition, as described above.

\subsection{Fast Test Predictions}
To achieve constant time prediction, we approximate the latent mean and variance of $\bff_*$ by applying the same SKI technique.
In particular, for a set of $n_*$ testing points, $X_*$, we have
\begin{eqnarray}
    \Expec[\bff_*]
    & = & \muv_{X_*} + K_{\Xv_*,\Xv}\left[K_{X,X} + \sigma^2I\right]^{-1}\yv \nonumber\\
    & \approx & \muv_{X_*} + \tilde K_{\Xv_*,\Xv}\left[\tilde K_{X,X} + \sigma^2I\right]^{-1}\yv,
\end{eqnarray}
where $\tilde K_{X,X} = W K_{U,U} W^\top$ and $\tilde K_{\Xv_*,\Xv} = W_* K_{U,U} W^\top$, and $W$ and $W_*$ are $n \times m$ and $n_* \times m$ sparse interpolation matrices, respectively.
Since $K_{U,U} W^\top [\tilde K_{X,X} + \sigma^2\Iv]^{-1}\yv$ is precomputed at training time, at test time, we only multiply the latter with $W_*$ matrix which results which costs $\Ocal(n_*)$ operations leading to $\Ocal(1)$ operations per test point.
Similarly, approximate predictive variance can be also estimated in $\Ocal(1)$ operations~\citep{wilson2015msgp}.

Note that the fast prediction methodology can be readily applied to any trained Gaussian process model as it is agnostic to the way inference and learning were performed.

\section{Gradients for GPs with recurrent kernels}\label{sec:grad-comp}
GPs with deep recurrent kernels are trained by minimizing the negative log marginal likelihood objective function.
Below we derive the update rules.

By applying the chain rule, we get the following first order derivatives:
\begin{equation}
  \label{eq:lik_derivs}
  \frac{\partial\lik}{\partial\gamma} = \frac{\partial\lik}{\partial K}\cdot\frac{\partial K}{\partial\gammav},\quad \frac{\partial\lik}{\partial W} = \frac{\partial\lik}{\partial K}\cdot\frac{\partial K}{\partial\phiv}\cdot\frac{\partial\phiv}{\partial W}.
\end{equation}
The derivative of the log marginal likelihood w.r.t. to the kernel hyperparameters, $\thetav$, and the parameters of the recurrent map, $W$, are generic and take the following form~\citep[Ch.~5, Eq.~5.9]{williams2006gaussian}:
\begin{eqnarray}
  \frac{\partial\lik}{\partial \thetav} & = \frac{1}{2} \tr{\left[K^{-1} \yv\yv^\top K^{-1} - K^{-1}\right]\frac{\partial K}{\partial \thetav}},
  \label{eq:dlik_dtheta}\\
  \frac{\partial\lik}{\partial W} & = \frac{1}{2} \tr{\left[K^{-1} \yv\yv^\top K^{-1} - K^{-1}\right]\frac{\partial K}{\partial W}}.
  \label{eq:dlik_dw_app}
\end{eqnarray}
The derivative $\partial K / \partial \thetav$ is also standard and depends on the form of a particular chosen kernel function, $k(\cdot, \cdot)$.
However, computing each part of $\partial \lik / \partial W$ is a bit subtle, and hence we elaborate these derivations below.

Consider the $ij$-th entry of the kernel matrix, $K_{ij}$.
We can think of $K$ as a matrix-valued function of all the data vectors in $d$-dimensional transformed space which we denote by $H \in \RR^{N \times d}$.
Then $K_{ij}$ is a scalar-valued function of $H$ and its derivative w.r.t. the $l$-th parameter of the recurrent map, $W_l$, can be written as follows:
\begin{equation}
  \label{eq:dcov_dw}
  \frac{\partial K_{ij}}{\partial W_l} = \tr{\left(\frac{\partial K_{ij}}{\partial H}\right)^\top \frac{\partial H}{\partial W_l}}.
\end{equation}
Notice that $\partial K_{ij} / \partial H$ is a derivative of a scalar w.r.t. to a matrix and hence is a matrix; $\partial H / \partial W_l$ is a derivative of a matrix w.r.t. to a scalar which is taken element-wise and also gives a matrix.
Also notice that $K_{ij}$ is a function of $H$, but it only depends the $i$-th and $j$-th elements for which the kernel is being computed.
This means that $\partial K_{ij} / \partial H$ will have only non-zero $i$-th row and $j$-th column and allows us to re-write~\eqref{eq:dcov_dw} as follows:
\begin{equation}
\begin{aligned}
  \frac{\partial K_{ij}}{\partial W_l}
  & = \left(\frac{\partial K_{ij}}{\partial \hv_i}\right)^\top \frac{\partial \hv_i}{\partial W_l} + \left(\frac{\partial K_{ij}}{\partial \hv_j}\right)^\top \frac{\partial \hv_j}{\partial W_l}\\
  & = \left(\frac{\partial K(\hv_i, \hv_j)}{\partial \hv_i}\right)^\top \frac{\partial \hv_i}{\partial W_l} + \left(\frac{\partial K(\hv_i, \hv_j)}{\partial \hv_j}\right)^\top \frac{\partial \hv_j}{\partial W_l}.
\end{aligned}
\end{equation}
Since the kernel function has two arguments, the derivatives must be taken with respect of each of them and evaluated at the corresponding points in the hidden space, $\hv_i = \phiv(\xv_i)$ and $\hv_j = \phiv(\xv_j)$.
When we plug this into~\eqref{eq:dlik_dw_app}, we arrive at the following expression:
\begin{equation}
  \frac{\partial\lik}{\partial W_l} =\frac{1}{2} \sum_{i,j} \left(K^{-1} \yv\yv^\top K^{-1} - K^{-1}\right)_{ij}
  \left\{\left(\frac{\partial K(\hv_i, \hv_j)}{\partial \hv_i}\right)^\top \frac{\partial \hv_i}{\partial W_l} + \left(\frac{\partial K(\hv_i, \hv_j)}{\partial \hv_j}\right)^\top \frac{\partial \hv_j}{\partial W_l}\right\}.
\end{equation}
The same expression can be written in a more compact form using the \emph{Einstein notation}:
\begin{equation}
    \frac{\partial\lik}{\partial W_l} = \frac{1}{2} \left(K^{-1} \yv\yv^\top K^{-1} - K^{-1}\right)_i^j \left( \left[\frac{\partial K}{\partial \hv}\right]_i^{jd} + \left[\frac{\partial K}{\partial \hv}\right]_j^{id} \right) \left[\frac{\partial \hv}{\partial W}\right]_i^{dl}
\end{equation}
where $d$ indexes the dimensions of the $\hv$ and $l$ indexes the dimensions of $W$.

In practice, deriving a computationally efficient analytical form of $\partial K / \partial \hv$ might be too complicated for some kernels (\textit{e.g.}, the spectral mixture kernels~\citep{wilson2013gaussian}), especially if the grid-based approximations of the kernel are enabled.
In such cases, we can simply use a finite difference approximation of this derivative.
As we remark in the following section, numerical errors that result from this approximation do not affect convergence of the algorithm.

\section{Convergence results}\label{sec:convergence}

Convergence results for the semi-stochastic alternating gradient schemes with and without delayed kernel matrix updates are based on~\citep{xu2015block}.
There are a few notable differences between the original setting and the one considered in this paper:
\begin{enumerate}
    \item \citet{xu2015block} consider a stochastic program that minimizes the expectation of the objective w.r.t. some distribution underlying the data:
    \begin{equation}
        \label{eq:xu_yin_stochastic_program}
        \min_{x \in \Xc} f(x) := \Expec_{\xi} F(x; \xi),
    \end{equation}
    where every iteration a new $\xi$ is sampled from the underlying distribution.
    In our case, the goal is to minimize the negative log marginal likelihood on a particular given dataset.
    This is equivalent to the original formulation \eqref{eq:xu_yin_stochastic_program}, but with the expectation taken w.r.t. the empirical distribution that corresponds to the given dataset.
    \item The optimization procedure of \citet{xu2015block} has access to only a single random point generated from the data distribution at each step.
    Our algorithm requires having access to the entire training data each time the kernel matrix is computed.
    \item For a given sample, \citet{xu2015block} propose to loop over a number of coordinate blocks and apply Gauss–Seidel type gradient updates to each block.
    Our semi-stochastic scheme has only two parameter blocks, $\thetav$ and $W$, where $\thetav$ is updated deterministically on the entire dataset while $W$ is updated with stochastic gradient on samples from the empirical distribution.
\end{enumerate}
Noting these differences, we first adapt convergence results for the smooth non-convex case~\cite[Theorem 2.10]{xu2015block} to our scenario, and then consider the variant with delaying kernel matrix updates.

\subsection{Semi-stochastic alternating gradient}\label{sec:convergence-alg-1}

As shown in Algorithm~\ref{alg:semi-stoch-alt-grad}, we alternate between updating $\thetav$ and $W$.
At step $t$, we get a mini-batch of size $N_t$, $\xv_t \equiv \{\bar\xv_i\}_{i \in \Ical_t}$, which is just a selection of points from the full set, $\xv$.
Define the gradient errors for $\thetav$ and $W$ at step $t$ as follows:
\begin{equation}
    \label{eq:grad-error}
    \bdelta_{\thetav}^t := \tilde\gv_{\thetav}^t - \gv_{\thetav}^t, \quad \bdelta_{W}^t := \tilde\gv_{W}^t - \gv_{W}^t,
\end{equation}
where $\gv^\Tv_{\thetav}$ and $\gv^\Tv_{W}$ are the true gradients and $\tilde\gv^\Tv_{\thetav}$ and $\tilde\gv^\Tv_{W}$ are estimates\footnote{Here, we consider the gradients and their estimates scaled by the number of full data points, $N$, and the mini-batch size, $N_t$, respectively. These constant scaling is introduced for sake of having cleaner proofs.}.
We first update $\thetav$ and then $W$, and hence the expressions for the gradients take the following form:
\begin{eqnarray}
    \tilde\gv_{\thetav}^t \equiv \gv_{\thetav}^t & = & \frac{1}{N} \nabla_{\thetav} \lik(\thetav^t, W^t) = \frac{1}{2N} \sum_{i,j} \left(K_t^{-1} \yv\yv^\top K_t^{-1} - K_t^{-1}\right)_{ij} \left(\frac{\partial K_t}{\partial \thetav}\right)_{ij} \\
    \gv_{W}^t & = & \frac{1}{N} \nabla_{W} \lik(\thetav^{t+1}, W^t) = \frac{1}{2N} \sum_{i,j} \left(K_{t+1}^{-1} \yv\yv^\top K_{t+1}^{-1} - K_{t+1}^{-1}\right)_{ij} \left(\frac{\partial K_{t+1}}{\partial W}\right)_{ij} \\
    \tilde\gv_{W}^t & = & \frac{1}{2N_t} \sum_{i,j \in \Ical_t} \left(K_{t+1}^{-1} \yv\yv^\top K_{t+1}^{-1} - K_{t+1}^{-1}\right)_{ij} \left(\frac{\partial K_{t+1}}{\partial W}\right)_{ij}
\end{eqnarray}
Note that as we have shown in Appendix~\ref{sec:grad-comp}, when the kernel matrix is fixed, $\gv_{W}$ and $\tilde\gv_{W}$ factorize over $\xv$ and $\xv_t$, respectively.
Further, we denote all the mini-batches sampled before $t$ as $\xv_{[t-1]}$.

\setcounter{theorem}{0}
\begin{lemma}\label{lem:unbiased-grad}
   For any step $t$, $\Expec[\bdelta_{W}^t \mid \xv_{[t-1]}] = \Expec[\bdelta_{\thetav}^t \mid \xv_{[t-1]}] = 0$.
\end{lemma}
\begin{proof}
  First, $\bdelta_{\thetav}^t \equiv 0$, and hence $\Expec[\bdelta_{\thetav}^t \mid \xv_{[t-1]}] = 0$ is trivial.
  Next, by definition of $\bdelta_{W}^t$, we have $\Expec[\bdelta_{W}^t \mid \xv_{[t-1]}] = \Expec[\tilde\gv_{W}^t - \gv_{W}^t \mid \xv_{[t-1]}]$.
  Consider the following:

  \begin{itemize}
    \item Consider $\Expec[\gv_{W}^t \mid \xv_{[t-1]}]$: $\gv_{W}^t$ is a deterministic function of $\thetav^{t+1}$ and $W^t$.
    $\thetav^{t+1}$ is being updated deterministically using $\gv_{\thetav}^{t-1}$, and hence only depends on $W^t$ and $\thetav^t$.
    Therefore, it is independent of $\xv_t$, which means that $\Expec[\gv_{W}^t \mid \xv_{[t-1]}] \equiv \gv_{W}^t$.

    \item Now, consider $\Expec[\tilde\gv_{W}^t \mid \xv_{[t-1]}]$:
    Noting that the expectation is taken w.r.t. the empirical distribution and $K_{t+1}$ does not depend on the current mini-batch, we can write:
    \begin{equation}
    \begin{aligned}
      \Expec[\tilde\gv_{W}^t \mid \xv_{[t-1]}]
      &= \Expec\left[\frac{1}{2N_t} \sum_{i,j \in \Ical_t} \left(K_{t+1}^{-1} \yv\yv^\top K_{t+1}^{-1} - K_{t+1}^{-1}\right)_{ij} \left(\frac{\partial K_{t+1}}{\partial W}\right)_{ij}\right] \\
      &= \frac{N_t}{2 N N_t} \sum_{i,j} \left(K_{t+1}^{-1} \yv\yv^\top K_{t+1}^{-1} - K_{t+1}^{-1}\right)_{ij} \left(\frac{\partial K_{t+1}}{\partial W}\right)_{ij} \\
      &\equiv \gv_{W}^t.
    \end{aligned}
    \end{equation}
  \end{itemize}
  Finally, $\Expec[\bdelta_{W}^t \mid \xv_{[t-1]}] = \gv_{W}^t - \gv_{W}^t = 0$.
\end{proof}

In other words, semi-stochastic gradient descent that alternates between updating $\thetav$ and $W$ computes unbiased estimates of the gradients on each step.

\setcounter{theorem}{0}
\begin{remark}
  Note that in case if $(\partial K_{t+1} / \partial W)$ is computed approximately, Lemma~\ref{lem:unbiased-grad} still holds since both $\gv_{W}^t$ and $\Expec[\tilde\gv_{W}^t \mid \xv_{[t-1]}]$ will contain exactly the same numerical errors.
\end{remark}

\setcounter{theorem}{0}
\begin{assumption}
    \label{ass:A1}
    For any step $t$, $\Expec\|\bdelta_{\thetav}^t\|^2 \leq \sigma_t^2$ and $\Expec\|\bdelta_{W}^t\|^2 \leq \sigma_t^2$.
\end{assumption}

Lemma~\ref{lem:unbiased-grad} and Assumption~\ref{ass:A1} result into a stronger condition than the original assumption given by \citet{xu2015block}.
This is due to the semi-stochastic nature of the algorithm, it simplifies the analysis, though it is not critical.
Assumptions~\ref{ass:A2} and~\ref{ass:A3} are straightforwardly adapted from the original paper.

\begin{assumption}
    \label{ass:A2}
    The objective function, $\lik$, is lower bounded and its partial derivatives w.r.t. $\thetav$ and $W$ are uniformly Lipschitz with constant $L > 0$:
    \begin{equation}
        \|\nabla_{\thetav}\lik(\thetav, W) - \nabla_{\thetav}\lik(\tilde\thetav, W)\| \leq L \|\thetav - \tilde\thetav\|,
        \quad
        \|\nabla_{W}\lik(\thetav, W) - \nabla_{W}\lik(\thetav, \tilde W)\| \leq L \|W - \tilde W\|.
    \end{equation}
\end{assumption}

\begin{assumption}
    \label{ass:A3}
    There exists a constant $\rho$ such that $\|\thetav^t\|^2 + \Expec\|W^t\|^2 \leq \rho^2$ for all $t$.
\end{assumption}

\setcounter{theorem}{1}
\begin{lemma}\label{lem:expec-grad-bound}
    Under Assumptions~\ref{ass:A2} and~\ref{ass:A3},
    \begin{equation*}
        \Expec\|W^t\|^2 \leq \rho^2, \quad
        \Expec\|\thetav^t\|^2 \leq \rho^2, \quad
        \Expec\|\gv_{W}^t\|^2 \leq M_{\rho}^2, \quad
        \Expec\|\gv_{\thetav}^t\|^2 \leq M_{\rho}^2 \quad \forall t,
    \end{equation*}
    where $M_{\rho}^2 = 4L^2\rho^2 + 2 \max \{\nabla_{\thetav} \lik(\thetav^0, W^0), \nabla_{W} \lik(\thetav^0, W^0)\}$.
\end{lemma}
\begin{proof}
    The inequalities for $\Expec\|\thetav\|^2$ and $\Expec\|W\|^2$ are trivial, and the ones for $\Expec\|\gv_{\thetav}^t\|^2$ and $\Expec\|\gv_{W}^t\|^2$ are merely corollaries from Assumption~\ref{ass:A2}.
\end{proof}

Negative log marginal likelihood of a Gaussian process with a structured kernel is a nonconvex function of its arguments.
Therefore, we can only show that the algorithm converges to a stationary point, \emph{i.e.}, a point at which the gradient of the objective is zero.

\setcounter{theorem}{0}
\begin{theorem}\label{thm:alg-1-convergence}
Let $\{(\thetav^t, W^t)\}$ be a sequence generated from Algorithm~\ref{alg:semi-stoch-alt-grad} with learning rates for $\thetav$ and $W$ being $\lambda_{\thetav}^t = c_{\thetav}^t \alpha_t$ and $\lambda_{W}^t = c_{W}^t \alpha_t$, respectively, where $c_{\thetav}^t, c_{W}^t, \alpha_t$ are positive scalars, such that
\begin{equation}\label{eq:thm1-conditions}
    0 < \inf_t c_{\{\thetav,W\}}^t \leq \sup_t c_{\{\thetav,W\}}^t < \infty, \quad
    \alpha_t \geq \alpha_{t+1}, \quad
    \sum_{t=1}^{\infty} \alpha_t = +\infty, \quad
    \sum_{t=1}^{\infty} \alpha_t^2 < +\infty, \quad
    \forall t.
\end{equation}
Then, under Assumptions~\ref{ass:A1} through~\ref{ass:A3} and if $\sigma = \sup_t \sigma_t < \infty$, we have
\begin{equation}
    \lim_{t \rightarrow \infty} \Expec\|\nabla\lik(\thetav^t, W^t)\| = 0.
\end{equation}
\end{theorem}
\begin{proof}
  The proof is an adaptation of the one given by~\citet{xu2015block} with the following three adjustments: (1) we have only two blocks of coordinates and (2) updates for $\thetav$ are deterministic which zeros out their variance terms, (3) the stochastic gradients are unbiased.
  \begin{itemize}
    \item From the Lipschitz continuity of $\nabla_{W} \lik$ and $\nabla_{\thetav} \lik$ (Assumption~\ref{ass:A2}), we have:
    \begin{eqnarray*}
        && \lik(\thetav^{t+1}, W^{t+1}) - \lik(\thetav^{t+1}, W^{t}) \\
        & \leq & \langle \gv_{W}^t, W^{t+1} - W^{t}\rangle + \frac{L}{2} \|W^{t+1} - W^{t}\|^2 \\
        & = & -\lambda_{W}^t \langle \gv_{W}^t, \tilde\gv_{W}^t \rangle + \frac{L}{2}(\lambda_{W}^t)^2 \|\tilde\gv_{W}^t\|^2 \\
        & = & -\left(\lambda_{W}^t - \frac{L}{2}(\lambda_{W}^t)^2\right)\|\gv_{W}^t\|^2 +  \frac{L}{2}(\lambda_{W}^t)^2 \|\bdelta_{W}^t\|^2 - \left(\lambda_{W}^t - L(\lambda_{W}^t)^2\right) \langle \gv_{W}^t, \bdelta_{W}^t\rangle \\
        & = & -\left(\lambda_{W}^t - \frac{L}{2}(\lambda_{W}^t)^2\right)\|\gv_{W}^t\|^2 + \frac{L}{2}(\lambda_{W}^t)^2 \|\bdelta_{W}^t\|^2\\
        && - \left(\lambda_{W}^t - L(\lambda_{W}^t)^2\right) \left( \langle \gv_{W}^t - \nabla_{W} \lik (\thetav^t, W^t), \bdelta_{W}^t\rangle + \langle\nabla_{W} \lik (\thetav^t, W^t), \bdelta_{W}^t\rangle \right) \\
        & \leq & -\left(\lambda_{W}^t - \frac{L}{2}(\lambda_{W}^t)^2\right)\|\gv_{W}^t\|^2 + \frac{L}{2}(\lambda_{W}^t)^2 \|\bdelta_{W}^t\|^2 - \left(\lambda_{W}^t - L(\lambda_{W}^t)^2\right) \langle\nabla_{W} \lik (\thetav^t, W^t), \bdelta_{W}^t\rangle \\
        && + \frac{L}{2} \lambda_{\thetav}^t \left( \lambda_{W}^t + L(\lambda_{W}^t)^2 \right) (\|\bdelta_{W}^t\|^2 + \|\gv_{\thetav}^t\|^2),
    \end{eqnarray*}
    where the last inequality comes from the following (note that $\gv_{W}^t := \nabla_{W} \lik (\thetav^{t+1}, W^t)$):
    \begin{eqnarray*}
      && - \left(\lambda_{W}^t - L(\lambda_{W}^t)^2\right) \langle \nabla_{W} \lik (\thetav^{t+1}, W^t) - \nabla_{W} \lik (\thetav^t, W^t), \bdelta_{W}^t\rangle \\
      & \leq & \left|\lambda_{W}^t - L(\lambda_{W}^t)^2\right|\|\bdelta_{W}^t\| \|\nabla_{W} \lik (\thetav^{t+1}, W^t) - \nabla_{W} \lik (\thetav^t, W^t)\| \\
      & \leq & L \lambda_{\thetav}^t \left|\lambda_{W}^t - L(\lambda_{W}^t)^2\right|\|\bdelta_{W}^t\| \|\gv_{\thetav}^t\| \\
      & \leq & \frac{L}{2} \lambda_{\thetav}^t \left(\lambda_{W}^t + L(\lambda_{W}^t)^2\right) (\|\bdelta_{W}^t\|^2 + \|\gv_{\thetav}^t\|^2).
    \end{eqnarray*}
    Analogously, we derive a bound on $\lik(\thetav^{t+1}, W^t) - \lik(\thetav^t, W^t)$:
    \begin{eqnarray*}
        && \lik(\thetav^{t+1}, W^t) - \lik(\thetav^t, W^t) \\
        & \leq & -\left(\lambda_{\thetav}^t - \frac{L}{2}(\lambda_{\thetav}^t)^2\right)\|\gv_{\thetav}^t\|^2 + \frac{L}{2}(\lambda_{\thetav}^t)^2 \|\bdelta_{\thetav}^t\|^2 - \left(\lambda_{\thetav}^t - L(\lambda_{\thetav}^t)^2\right) \langle\nabla_{\thetav} \lik (\thetav^t, W^t), \bdelta_{\thetav}^t\rangle
    \end{eqnarray*}

    \item Since $(\thetav^t, W^t)$ is independent from the current mini-batch, $\xv_t$, using Lemma~\ref{lem:unbiased-grad}, we have that $\Expec\langle\nabla_{\thetav}\lik(\thetav^t, W^t), \bdelta_{\thetav}^t\rangle = 0$ and $\Expec\langle\nabla_{W}\lik(\thetav^t, W^t), \bdelta_{W}^t\rangle = 0$.

    Summing the two obtained inequalities and applying Assumption~\ref{ass:A1}, we get:
    \begin{eqnarray*}
        && \Expec[\lik(\thetav^{t+1}, W^{t+1}) - \lik(\thetav^t, W^t)] \\
        & \leq & -\left(\lambda_{W}^t - \frac{L}{2}(\lambda_{W}^t)^2\right)\Expec\|\gv_{W}^t\|^2 + \frac{L}{2}(\lambda_{W}^t)^2 \sigma + \frac{L}{2} \lambda_{\thetav}^t \left(\lambda_{W}^t + L(\lambda_{W}^t)^2\right) (\sigma + \Expec\|\gv_{\thetav}^t\|^2) \\
        && -\left(\lambda_{\thetav}^t - \frac{L}{2}(\lambda_{\thetav}^t)^2\right)\Expec\|\gv_{\thetav}^t\|^2 + \frac{L}{2}(\lambda_{\thetav}^t)^2 \sigma \\
        & \leq & -\left(c \alpha_t - \frac{L}{2} C^2 \alpha_t^2\right)\Expec\|\gv_{W}^t\|^2 + L C^2 \sigma \alpha_t + \frac{L}{2} C^2 \alpha_t^2 \left(1 + LC\alpha_t\right) (\sigma + \Expec\|\gv_{\thetav}^t\|^2) \\
        && -\left(c \alpha_t - \frac{L}{2}C^2\alpha_t^2\right)\Expec\|\gv_{\thetav}^t\|^2 + \frac{L}{2}C^2\alpha_t^2 \sigma,
    \end{eqnarray*}
    where we denoted $c = \min\{\inf_t c_{W}^t, \inf_t c_{\thetav}^t\}$, $C = \max\{\sup_t c_{W}^t, \sup_t c_{\thetav}^t\}$.
    Note that from Lemma~\ref{lem:expec-grad-bound}, we also have $\Expec\|\gv_{\thetav}^t\|^2 \leq M_{\rho}^2$ and $\Expec\|\gv_{W}^t\|^2 \leq M_{\rho}^2$.
    Therefore, summing the right hand side of the final inequality over $t$ and using~\eqref{eq:thm1-conditions}, we have:
    \begin{equation}
        \lim_{t \rightarrow \infty} \Expec\lik(\thetav^{t+1}, W^{t+1}) - \Expec\lik(\thetav^0, W^0) \leq - c \sum_{t=1}^{\infty} \alpha_t \left(\Expec\|\gv_{\thetav}^t\|^2 + \Expec\|\gv_{W}^t\|^2\right).
    \end{equation}
    Since the objective function is lower bounded, this effectively means:
    \begin{equation*}
      \sum_{t=1}^{\infty} \alpha_t \Expec\|\gv_{\thetav}^t\|^2 < \infty, \quad
      \sum_{t=1}^{\infty} \alpha_t \Expec\|\gv_{W}^t\|^2 < \infty.
    \end{equation*}

    \item Finally, using Lemma~\ref{lem:expec-grad-bound}, our assumptions and Jensen's inequality, it follows that
    \begin{equation*}
        \left|\Expec\|\gv_{W}^{t+1}\|^2 - \Expec\|\gv_{W}^t\|^2\right|
        \leq 2 L M_{\rho} C \alpha_t \sqrt{2(M_{\rho}^2 + \sigma^2)}.
    \end{equation*}

  \end{itemize}
  According to Proposition~1.2.4 of~\citep{bertsekas1999nonlinear}, we have $\Expec\|\gv_{\thetav}^t\|^2 \rightarrow 0$ and $\Expec\|\gv_{W}^t\|^2 \rightarrow 0$ as $t \rightarrow \infty$, and hence
  \begin{eqnarray*}
    \Expec\|\nabla \lik (\thetav^t, W^t)\|
    & \leq & \Expec\|\nabla_{W} \lik (\thetav^t, W^t) - \gv_{\thetav}^t\| + \Expec\|\gv_{\thetav}^t\| + \Expec\|\nabla_{W} \lik (\thetav^t, W^t) - \gv_{\thetav}^t\| + \Expec\|\gv_{W}^t\|\\
    & \leq & 2LC\sqrt{2(M_{\rho}^2 + \sigma^2)}\alpha + \Expec\|\gv_{\thetav}^t\| + \Expec\|\gv_{W}^t\| \rightarrow 0 \text{ as } t \rightarrow \infty,
  \end{eqnarray*}
  where the first term of the last inequality follows from Lemma~\ref{lem:expec-grad-bound} and Jensen's inequality.
\end{proof}

\subsection{Semi-stochastic gradient with delayed kernel matrix updates}
\label{sec:convergence-alg-2}

We show that given a bounded delay on the kernel matrix updates, the algorithm is still convergent.
Our analysis is based on computing the change in $\bdelta_{\thetav}^t$ and $\bdelta_{W}^t$ and applying the same argument as in Theorem~\ref{thm:alg-1-convergence}.
The only difference is that we need to take into account the perturbations of the kernel matrix due to the introduced delays, and hence we have to impose certain assumptions on its spectrum.

\setcounter{theorem}{3}
\begin{assumption}\label{ass:nn-lipschitz}\label{ass:A4}
    Recurrent transformations, $\phiv_{W}(\bar\xv)$, is L-Lipschitz w.r.t. $W$ for all $\bar\xv \in \Xcal^L$:
    \begin{equation*}
        \|\phiv_{\tilde W}(\bar\xv) - \phiv_{W}(\bar\xv)\| \leq L \|\tilde W - W\|.
    \end{equation*}
\end{assumption}

\begin{assumption}\label{ass:kernel-lipschitz}\label{ass:A5}
    The kernel function, $k(\cdot, \cdot)$, is uniformly G-Lipschitz and its first partial derivatives are uniformly J-Lipschitz:
    \begin{equation*}
        \begin{aligned}
            \left\|k(\tilde\hv_1, \hv_2) - k(\hv_1, \hv_2) \right\| \leq G \|\tilde\hv_1 - \hv_1\|,\,
            \left\|k(\hv_1, \tilde\hv_2) - k(\hv_1, \hv_2) \right\| \leq G \|\tilde\hv_2 - \hv_2\|,
            \\
            \left\|\partial_1 k(\hv_1, \hv_2) - \partial_1 k(\tilde\hv_1, \hv_2) \right\| \leq J \|\tilde\hv_1 - \hv_1\|,\,
            \left\|\partial_2 k(\hv_1, \hv_2) - \partial_2 k(\hv_1, \tilde\hv_2) \right\| \leq J \|\tilde\hv_2 - \hv_2\|.
        \end{aligned}
    \end{equation*}
\end{assumption}

\begin{assumption}\label{ass:kernel-spec}\label{ass:A6}
    For any collection of data representations, $\{\hv_i\}_{i=1}^N$, the smallest eigenvalue of the corresponding kernel matrix, $K$, is lower bounded by a positive constant $\gamma > 0$.
\end{assumption}

Note that not only the assumptions are relevant to practice, Assumptions~\ref{ass:kernel-lipschitz} and~\ref{ass:kernel-spec} can be also controlled by choosing the class of kernel functions used in the model.
For example, the smallest eigenvalue of the kernel matrix, $K$, can be controlled by the smoothing properties of the kernel~\citep{williams2006gaussian}.

Consider a particular stochastic step of Algorithm~\ref{alg:semi-stoch-grad-stale-kernel} at time $t$ for a given mini-batch, $\xv_t$, assuming that the kernel was last updated $\tau$ steps ago.
The stochastic gradient will take the following form:
\begin{equation}
    \hat\gv_{W}^t = \frac{1}{N_t} \nabla_{W} \lik(\thetav^{t+1}, W^t, K_{t-\tau}) = \frac{1}{2N_t} \sum_{i,j \in \Ical_t} \left(K_{t-\tau}^{-1} \yv\yv^\top K_{t-\tau}^{-1} - K_{t-\tau}^{-1}\right)_{ij} \left(\frac{\partial K_{t-\tau}}{\partial W}\right)_{ij}.
\end{equation}
We can define $\hat\bdelta_{W}^t = \hat\gv_{W}^t - \gv_{W}^t$ and uniformly bound $\|\hat\bdelta_{W}^t - \bdelta_{W}^t\|$ in order to enable the same argument as in Theorem~\ref{thm:alg-1-convergence}.
To do that, we simply need to understand effect of perturbation of the kernel matrix on $\tilde\gv_{W}^t$.

\setcounter{theorem}{2}
\begin{lemma}\label{lem:kernel-perturb-bound}
    Under the given assumptions, the following bound holds for all $i, j = 1, \dots, N$:
    \begin{equation}
        \left|\left(K_{t-\tau}^{-1} \yv\yv^\top K_{t-\tau}^{-1} - K_{t-\tau}^{-1}\right)_{ij} - \left(K_t^{-1} \yv\yv^\top K_t^{-1} - K_t^{-1}\right)_{ij}\right| \leq D^2 \|y\|^2 + D,
    \end{equation}
    where $D = \gamma^{-1} + \left(\gamma - 2 G L \tau \lambda \sigma \sqrt{N}\right)^{-1}$.
\end{lemma}
\begin{proof}
The difference between $K_{t-\tau}^{-1}$ and $K_t^{-1}$ is simply due to that the former has been computed for $W^{t-\tau}$ and the latter for $W^t$.
To prove the bound, we need multiple steps.
\begin{itemize}
    \item First, we need to bound the element-wise difference between $K_{t-\tau}$ and $K_t$.
    This is done by using Assumptions~\ref{ass:nn-lipschitz} and~\ref{ass:kernel-lipschitz} and the triangular inequality:
    \begin{eqnarray*}
        |(K_t)_{ij} - (K_{t-\tau})_{ij}|
        & \leq & G\left(\|\phiv_{W^t}(\bar\xv_i) - \phiv_{W^{t-\tau}}(\bar\xv_i)\| + \|\phiv_{W^t}(\bar\xv_j) - \phiv_{W^{t-\tau}}(\bar\xv_j)\|\right) \\
        & \leq & 2 G L \|W^t - W^{t - \tau}\| \\
        & = & 2 G L \|\sum_{s=1}^{\tau} \lambda_{W}^{t - \tau + s} \hat\gv_{W}^{t - \tau + s}\| \\
        & \leq & 2 G L \tau \lambda \sigma
    \end{eqnarray*}

    \item Next, since each element of the perturbed matrix is bounded by $2 G L \tau \lambda \sigma$, we can bound its spectral norm as follows:
    \begin{equation*}
        \|K_t - K_{t-\tau}\| \leq \|K_t - K_{t-\tau}\|_F \leq 2 G L \tau \lambda \sigma \sqrt{N},
    \end{equation*}
    which means that the minimal singular value of the perturbed matrix is at least $\sigma_1 = \gamma - 2 G L \tau \lambda \sigma \sqrt{N}$ due to Assumption~\ref{ass:kernel-spec}.

    \item The spectral norm of the expression of interest can be bounded (quite pessimistically!) by summing up together the largest eigenvalues of the matrix inverses:
    \begin{equation*}
        \begin{aligned}
            & \left|\left(K_{t-\tau}^{-1} \yv\yv^\top K_{t-\tau}^{-1} - K_{t-\tau}^{-1}\right)_{ij} - \left(K_t^{-1} \yv\yv^\top K_t^{-1} - K_t^{-1}\right)_{ij}\right| \\
            \leq & \left\|\left(K_{t-\tau}^{-1} - K_t^{-1}\right)\yv\yv^\top \left(K_{t-\tau}^{-1} - K_t^{-1}\right) - \left(K_{t-\tau}^{-1} - K_t^{-1}\right)\right\| \\
            \leq & \left(\gamma^{-1} + \left(\gamma - 2 G L \tau \lambda \sigma \sqrt{N}\right)^{-1}\right)^2\|\yv\|^2 + \gamma^{-1} + \left(\gamma - 2 G L \tau \lambda \sigma \sqrt{N}\right)^{-1}.
        \end{aligned}
    \end{equation*}
\end{itemize}
Each element of a matrix is bounded by the largest eigenvalue.
\end{proof}

Using Lemma~\ref{lem:kernel-perturb-bound}, it is straightforward to extend Theorem~\ref{thm:alg-1-convergence} to Algorithm~\ref{alg:semi-stoch-grad-stale-kernel}.

\setcounter{theorem}{1}
\begin{theorem}\label{thm:alg-2-convergence}
Let $\{(\thetav^t, W^t)\}$ be a sequence generated from Algorithm~\ref{alg:semi-stoch-grad-stale-kernel} with learning rates for $\thetav$ and $W$ being $\lambda_{\thetav}^t = c_{\thetav}^t \alpha_t$ and $\lambda_{W}^t = c_{W}^t \alpha_t$, respectively, where $c_{\thetav}^t, c_{W}^t, \alpha_t$ are positive scalars, such that
\begin{equation}\label{eq:thm2-conditions}
    0 < \inf_t c_{\{\thetav,W\}}^t \leq \sup_t c_{\{\thetav,W\}}^t < \infty, \quad
    \alpha_t \geq \alpha_{t+1}, \quad
    \sum_{t=1}^{\infty} \alpha_t = +\infty, \quad
    \sum_{t=1}^{\infty} \alpha_t^2 < +\infty, \quad
    \forall t.
\end{equation}
Then, under Assumptions~\ref{ass:A1} through~\ref{ass:A6} and if $\sigma = \sup_t \sigma_t < \infty$, we have
\begin{equation}
    \lim_{t \rightarrow \infty} \Expec\|\nabla\lik(\thetav^t, W^t)\| = 0.
\end{equation}
\end{theorem}
\begin{proof}
The proof is identical to the proof of Theorem~\ref{thm:alg-1-convergence}.
The only difference is in the following upper bound on the expected gradient error: $\Expec\|\bdelta_{W}^t\| \leq \sigma + 2 N D (D\|y\|^2 + 1) J \rho$, where $D$ is as given in Lemma~\ref{lem:kernel-perturb-bound}.
\end{proof}

\setcounter{theorem}{0}
\begin{remark}
    Even though the provided bounds are crude due to pessimistic estimates of the perturbed kernel matrix spectrum\footnote{Tighter bounds can be derived by inspecting the effects of perturbations on specific kernels as well as using more specific assumptions about the data distribution.}, we still see a fine balance between the delay, $\tau$, and the learning rate, $\lambda$, as given in the expression for the $D$ constant.
\end{remark}

\section{Details on the datasets}\label{sec:data-details}

The datasets varied in the number of time steps (from hundreds to a million), input and output dimensionality, and the nature of the estimation tasks.

\subsection{Self-driving car}
The following is description of the input and target time series used in each of the autonomous driving tasks (dimensionality is given in parenthesis).
\begin{itemize}
    \item \textbf{Car speed estimation:}
    \begin{itemize}
        \item \emph{Features:} GPS velocity (3), fiber gyroscope (3).
        \item \emph{Targets:} speed measurements from the car speedometer.
    \end{itemize}
    \item \textbf{Car yaw estimation:}
    \begin{itemize}
        \item \emph{Features:} acceleration (3), compass measurements (3).
        \item \emph{Targets:} yaw in the car-centric frame.
    \end{itemize}
    \item \textbf{Lane sequence prediction:} Each lane was represented by 8 cubic polynomial coefficients [4 coefficients for $x$ (front) and 4 coefficients for $y$ (left) axes in the car-centric frame].
    Instead of predicting the coefficients (which turned out to lead to overall less stable results), we discretized the lane curves using 7 points (initial, final and 5 equidistant intermediate points).
    \begin{itemize}
        \item \emph{Features:} lanes at a previous time point (16), GPS velocity (3), fiber gyroscope (3), compass (3), steering angle (1).
        \item \emph{Targets:} coordinates of the lane discretization points (7 points per lane resulting in 28 total output dimensions).
    \end{itemize}
    \item \textbf{Estimation of the nearest front vehicle position:} $(x, y)$ coordinates in the car-centric frame.
    \begin{itemize}
        \item \emph{Features:} $x$ and $y$ at a previous time point (2), GPS velocity (3), fiber gyroscope (3), compass (3), steering angle (1).
        \item \emph{Targets:} $x$ and $y$ coordinates.
    \end{itemize}
\end{itemize}

\section{Neural architectures}\label{sec:architecture-details}
Details on the best neural architectures used for each of the datasets are given in Table~\ref{tab:architectures}.

%!TEX root = ../16-498.tex

\begin{table*}[t!]
\centering
\caption{\small
Summary of the feedforward and recurrent neural architectures and the corresponding hyperparameters used in the experiments.
GP-based models used the same architectures as their non-GP counterparts.
Activations are given for the hidden units; vanilla neural nets used linear output activations.}
\label{tab:architectures}
\vspace{1ex}
\small
\begin{tabular}{llccccll}
\toprule
\textbf{Name} & \textbf{Data} & \textbf{Time lag} & \textbf{Layers} & \textbf{Units$^*$} & \textbf{Type} & \textbf{Regularizer$^{**}$} & \textbf{Optimizer}\\
\midrule
\multirow{6}{*}{NARX}
& Actuator      & 32    & 1     & 256
& \multirow{6}{*}{\texttt{ReLU}} & \multirow{6}{*}{\texttt{dropout}(0.5)}  & \multirow{6}{*}{\texttt{Adam}(0.01)} \\
& Drives        & 16    & 1     & 128
&                       &                                 & \\
& GEF-power     & 48    & 1     & 256
&                       &                                 & \\
& GEF-wind      & 48    & 1     & 16
&                       &                                 & \\
& Car           & 128   & 1     & 128
&                       &                                 & \\
\midrule
\multirow{6}{*}{RNN}
& Actuator      & 32    & 1     & 64
& \multirow{6}{*}{\texttt{tanh}} & \multirow{6}{*}{\makecell{\texttt{dropout}(0.25),\\\texttt{rec\_dropout}(0.05)}}  & \multirow{6}{*}{\texttt{Adam}(0.01)} \\
& Drives        & 16    & 1     & 64
&                       &                                 & \\
& GEF-power     & 48    & 1     & 16
&                       &                                 & \\
& GEF-wind      & 48    & 1     & 32
&                       &                                 & \\
& Car           & 128   & 1     & 128
&                       &                                 & \\
\midrule
\multirow{6}{*}{LSTM}
& Actuator      & 32    & 1     & 256
& \multirow{6}{*}{\texttt{tanh}} & \multirow{6}{*}{\makecell{\texttt{dropout}(0.25),\\\texttt{rec\_dropout}(0.05)}}  & \multirow{6}{*}{\texttt{Adam}(0.01)} \\
& Drives        & 16    & 1     & 128
&                       &                                 & \\
& GEF-power     & 48    & 2     & 256
&                       &                                 & \\
& GEF-wind      & 48    & 1     & 64
&                       &                                 & \\
& Car           & 128   & 2     & 64
&                       &                                 & \\
\bottomrule
\end{tabular}
\flushleft
$^*$Each layer consisted of the same number of units given in the table.\\
$^{**}$\texttt{rec\_dropout} denotes the dropout rate of the recurrent weights~\citep{gal2016theoretically}.
\end{table*}

\clearpage
\bibliography{references}

\end{document}